\documentclass[10pt,reqno]{article}
\usepackage{amsmath,amssymb,amsthm}
\usepackage{esint}
\usepackage{bm}
\usepackage{url}
\usepackage{subfigure}
\usepackage{graphicx,color, epstopdf}
\usepackage{algorithm}
\usepackage{algpseudocode}
\usepackage{algorithmicx}
\usepackage{hyperref}

\usepackage{booktabs,multirow}
\usepackage{hhline}
\usepackage{setspace}

\usepackage{caption}
\usepackage{hyphsubst}
\usepackage{caption}

\usepackage{float}
\usepackage{filecontents}
\usepackage{hyperref}

\hypersetup{colorlinks,citecolor=blue,linkcolor=blue, urlcolor=blue}
\usepackage{pdflscape}
\usepackage{breakurl}

\def\bp{{\bf p}}

\def\bu{{\bf u}}
\def\bv{{\bf v}}
\def\bw{{\bf w}}
\def\bx{{\bf x}}

\newcommand{\ds}{\displaystyle}
\newcommand{\RR}{\mathbb{R}}
\newcommand{\NN}{\mathbb{N}}
\newcommand{\vf}{\varphi}

\newcommand{\CC}{\mathbb{C}}

\DeclareMathAlphabet{\itbf}{OML}{cmm}{b}{it}

\newtheorem{thm}{Theorem}[section]

\newtheorem{lem}[thm]{Lemma}
\newtheorem{prop}[thm]{Proposition}
\newtheorem{defn}[thm]{Definition}
\newtheorem{rem}[thm]{Remark}
\newtheorem{example}[thm]{Example}

\numberwithin{equation}{section}

\newcommand{\email}[1]{\protect\href{mailto:#1}{#1}}

\newcommand{\pathfigures}{Figures/}
\graphicspath{{\pathfigures}}

\begin{document}

\title{Performance Analysis of Fractional Learning Algorithms 
}

\author{
Abdul Wahab\footnotemark[1] \footnotemark[2]
\and
Shujaat Khan\footnotemark[3]
\and 
Imran Naseem\footnotemark[4]\, \footnotemark[5]
\and
Jong Chul Ye\footnotemark[3]
}
\maketitle
\renewcommand{\thefootnote}{\fnsymbol{footnote}}
\footnotetext[1]{Corresponding Author. E-mail address:  \email{abdul.wahab@nu.edu.kz}.}
\footnotetext[2]{Department of Mathematics, School of Sciences and Humanities, Nazarbayev University, 53, Kabanbay Batyr Ave., 010000, Nur-Sultan, Kazakhstan (\email{abdul.wahab@nu.edu.kz)}.}
\footnotetext[3]{Bio-Imaging, Signal Processing and Learning Lab., Department of Bio and Brain Engineering, Korea Advanced Institute of Science and Technology (KAIST), 291 Daehak-ro, Yuseong-gu, 34141, Daejeon, South Korea (\email{shujaat@kaist.ac.kr, jong.ye@kaist.ac.kr}).}
\footnotetext[4]{School of Electrical, Electronic and Computer Engineering, The University of Western Australia, 35 Stirling Highway, Crawley, Western Australia 6009, Australia (\email{imran.naseem@uwa.edu.au)}.}
\footnotetext[5]{College of Engineering, Karachi Institute of Economics and Technology, Korangi Creek, 75190, Pakistan.}
\renewcommand{\thefootnote}{\arabic{footnote}}

\begin{abstract}

Fractional learning algorithms are trending in signal processing and adaptive filtering recently. However, it is unclear whether the proclaimed superiority over conventional algorithms is well-grounded or is a myth as their performance has never been extensively analyzed.  In this article, a rigorous analysis of fractional variants of the least mean squares and steepest descent algorithms is performed. Some critical schematic kinks in fractional learning algorithms are identified. Their origins and consequences on the performance of the learning algorithms are discussed and swift ready-witted remedies are proposed. Apposite numerical experiments are conducted to discuss the convergence and efficiency of the fractional learning algorithms in stochastic environments. 

\end{abstract}



\section{Introduction} \label{s:intro}

The least mean square (LMS) algorithms are of paramount importance in the field of signal processing since their emergence \cite{Widrow1, Widrow2, CLMS}. In particular, they are used profusely in adaptive filtering and signal analysis  \cite{Haykin, WidrowBook, LMSBook1, Bellanger}. The key aspects that make LMS algorithms attractive are their low complexity, stability, and an unbiased mean convergence to the so-called Wiener solution in stationary environments \cite{Slock}. Unfortunately, its rate of convergence depends on the eigenvalue spread of the correlation matrix of the input signal in non-stationary environments \cite{Widrow2, Haykin}. Accordingly,  many variant algorithms were proposed to achieve better performance by curtailing the influence of the spectral properties of the input signal correlation matrix; see, for instance, the LMS-Newton algorithm \cite{Diniz92}, transform-domain algorithm \cite{Narayan}, and affine projection algorithm \cite{Soni}. On the other hand, a desire for computationally simpler algorithms has also led to the development of many variants such as quantized-error algorithms \cite{Bermudez,QKLMS, LMS11} and normalized LMS algorithms \cite{Yassa, Slock, NLMS}. A decent list of these variant algorithms along with the details of their key features is provided in \cite[Ch. 4]{LMSBook1}.  We also refer to fairly recent survey articles \cite{Widrow2016, Diniz2016} on the history of adaptive filtering and the development of the LMS algorithms. 

Recently, so-called fractional LMS and relavent learning algorithms are trending and many variants have been proposed over the past decade. The central idea behind these algorithms is to replace the classical integer-order gradient with a fractional gradient to achieve better performance. The idea of using fractional derivatives in the LMS algorithms first appeared in \cite{Raja1} in the context of a system identification problem. Subsequently, the algorithm was applied for chaotic and nonstationary time series prediction \cite{shoaib2014modified},  parameter estimation of CARMA systems \cite{CARMA, C-CARMA}, active noise control systems \cite{Noise, Shah-noise}, Hammerstein nonlinear autoregressive systems \cite{chaudhary2015design,ARMAX}, and nonlinear Box-Jenkins systems \cite{Box-Jenkins, C-Box-Jenkins}. Many variant fractional LMS algorithms were also proposed such as fractional steepest descent approach \cite{pu2013fractional, C-TNNLS}, fractional normalized LMS \cite{chaudhary2019normalized, khan2019design}, bias compensated fractional normalized LMS \cite{Yin2019, chaudhary2019design, C-EPJP}, fractional filtered-X normalized LMS \cite{yin2019novel}, complex fractional LMS \cite{Shah, C-Shah}, momentum fractional LMS \cite{khan2019design1, zubair2018momentum},  and Volterra fractional LMS \cite{chaudhary2019fractional,chaudhary2018novel}.  

Unfortunately, there were two kinks in the initial designs of the fractional LMS algorithms: (1) The emergence of the complex outputs and, (2) the \emph{long-memory} characteristic of the fractional derivatives. The iterate-update rules involved a fractional power of the past iterates. Consequently, the fractional algorithms would render complex outputs whenever the intermediate iterates became negative. This led to the introduction of an absolute value in the iterate-update rule as a ready-witted remedy; see, for example, \cite{ARMAX, Noise, chaudhary2020innovative}. Secondly, the fractional derivatives were \emph{non-local} in nature because they were defined in terms of convolution integrals, unlike their integer-order counterparts.  Accordingly, the fractional LMS algorithms carried forward information of all the past states that created doubts about their ability to provide \emph{punctual geometric} information, or simply about their convergence to the optimal solution even in stationary environments. On the other hand, due to long-memory requirements, the computational cost was also high. This led to the developments of variable initial value and variable fractional-order schemes; see, for instance, \cite{cheng2017universal, cheng2018identification, Cheng2017,cui2017innovative}. The modified algorithms were used in many applied problems \cite{khan2019fractional, shah2018fractional,zhu2018fractional}. More recently, these algorithms were also used in neural network designs \cite{chen2019delay, sheng2019convolutional, wang2017fractional, wang2017convergence}. 

Despite their widespread use, the performance analysis of the fractional LMS algorithms was mostly done heuristically. The first attempt can be traced back to Pu et al. \cite{pu2013fractional} who performed the rate of convergence analysis of a fractional steepest descent algorithm (however, the algorithm suffered from the issue of complex outputs and their analysis relied on some detrimental approximations; see \cite{C-TNNLS}). Chen et al. \cite{chen2017study} and Wei et al. \cite{wei2020generalization, C-Franklin} studied variable initial value and fractional-order gradient methods and showed that, under the assumption of convergence, the solution converges to the Wiener solution in stationary cases. Chaudary et al. \cite{chaudhary2019design} and \cite{ISA} discussed the convergence of two variants of fractional LMS in non-stationary cases, however, both results had technical flaws (see, \cite{C-ISA, C-EPJP}). Bershad, Wen and So \cite{Bershad} performed extensive numerical simulations and established that the performance of fractional LMS algorithms was no better than the conventional LMS algorithms in stochastic cases. 

The main goal of this article is to rigorously analyze the performance of the fractional learning algorithms from both mathematical and numerical points of view. As a simple example, we consider a system identification problem for which we study two representative fractional algorithms. We aim to rigorously derive the update rules and discuss the underlying assumptions and schematic kinks of the fractional LMS algorithms. Moreover, we study the connection, if any, between the ordinary and fractional critical points and discuss the convergence of the algorithms. 

The rest of this article is arranged as follows. In Section \ref{s:form}, we introduce some preliminaries, a model problem, and two representative fractional algorithms considered in this study. In Section \ref{s:Kinks}, we perform rigorous mathematical analysis of the representative algorithms. We first derive the iterate-update rules and discuss their underlying assumptions. The schematic kinks in the fractional algorithms are identified and discussed from the mathematical and geometrical points of view. The convergence of the algorithms in stationary environments is also discussed. In Section \ref{s:Num}, some numerical experiments are conducted for performance analysis in stochastic environments. The article ends with a brief discussion and summary of the results in Section \ref{s:Conc}.

\section{Problem Formulation}\label{s:form}

The primary concern of this article is the performance analysis of fractional learning algorithms. For simplicity, we entertain two representative fractional LMS algorithms for a system identification problem. Accordingly, we feel it important to provide some preliminaries from fractional Calculus (Section \ref{ss:preliminaries}), give a brief introduction to the system identification problem (Section \ref{ss:SIP}), and introduce the iterate-update rules for the prototype fractional algorithms (Section \ref{ss:rules}). 

\subsection{Elements of Fractional Calculus}\label{ss:preliminaries}

In this article, $\mathbb{N}$, $\mathbb{Z}$, $\mathbb{R}$, and $\mathbb{C}$ represent the sets of natural numbers, integers, real numbers, and complex numbers, respectively. Moreover, for any $z=x+\sqrt{-1}y\in\mathbb{C}$, we denote the real and imaginary components of $z$ by $\Re\{z\}:=x$ and $\Im\{z\}:=y$, respectively. Further, for any vector $\bu\in\RR^N$, the quantity $(\bu)_i$ represents the $i-$th component of the vector $\bu$, for any $1\leq i\leq N\in \mathbb{N}$. Similarly, for any matrix $\mathbf{M}\in\RR^{N\times M}$, the $ij-$th entry is represented by $(\mathbf{M})_{ij}$, for $1\leq i\leq N\in \mathbb{N}$ and $1\leq j\leq M\in \mathbb{N}$. Finally, throughout this article, all the vector quantities are represented by lower-case bold letters, and all the matrices are represented by upper-case bold letters. 

\begin{defn}[Gamma Function]\label{Def1}
For any $z\in\mathbb{C}$ such that $\Re \{z\}>0$,  the Euler's gamma function, denoted by $\Gamma$, is  defined by
$$
\Gamma(z):=\int_0^\infty t^{z-1}e^{-t}dt.
$$
Note that, $\Gamma(1+z)=z\Gamma(z)$ and therefore, $\Gamma(1+n)=n!$ for any $n\in\mathbb{Z}$ such that $n\geq 0$. 
\end{defn}

\begin{defn}[Riemann-Liouville Fractional Integrals {\cite[Eqs. (2.1.1)-(2.1.2)]{Kilbas}}]\label{Def2}

Let $z\in\CC$ such that $0<\Re\{z\}<1$. The left fractional integral of order $z$ over an interval $(a,b)\subset \RR$ of a left integrable function $\vf$, with $a<b$, is defined by 
\begin{align}
I_L^z[\vf](t):=\dfrac{1}{\Gamma(z)}\int_a^t\dfrac{\vf(\tau)}{(t-\tau)^{1-z}}d\tau,\quad \forall t\in(a,b).
\end{align}
Similarly,  the right fractional integral of order $z$ over $(a,b)\subset \RR$ of a right integrable function $\vf$ is defined by 
\begin{align}
I_R^z[\vf](t):=\dfrac{1}{\Gamma(z)}\int_t^b\dfrac{\vf(\tau)}{(\tau-t)^{1-z}}d\tau,\quad \forall t\in(a,b).
\end{align}
\end{defn}

\begin{defn}[Riemann-Liouville Fractional Derivatives {\cite[Eqs. (2.1.5)-(2.1.6)]{Kilbas}}]\label{def3}

Let $z\in\CC$ such that $m-1<\Re\{z\}<m$ for some $m\in\mathbb{N}$. The \emph{left Riemann-Liouville fractional derivative} of order $z$ over an interval $(a,b)$ with $a<b$ of a sufficiently smooth function $\vf$ is defined by 
\begin{align}
{^L_aD_t^z}[\vf] 
:=\dfrac{d^m}{dt^m}\left(I_L^{m-z}\big[\vf\big](t)\right)
=\dfrac{1}{\Gamma(m-z)}\dfrac{d^m}{dt^m}\int_a^t\dfrac{\vf(\tau)}{(t-\tau)^{1-m+z}}d\tau,\quad \forall t\in(a,b).\label{LRL}
\end{align}
Similarly, the \emph{right Riemann-Liouville fractional derivative} of $\vf$ of order $z$ is defined by
\begin{align}
{^R_aD_t^z}[\vf] 
:=(-1)^m\dfrac{d^m}{dt^m}\left(I_R^{m-z}\big[\vf\big](t)\right)
=\dfrac{(-1)^m}{\Gamma(m-z)}\dfrac{d^m}{dt^m}\int_t^b\dfrac{\vf(\tau)}{(\tau-t)^{1-m+z}}d\tau,\quad \forall t\in(a,b).\label{RRL}
\end{align}
\end{defn}

\begin{rem}\label{Rem1}
If $z\in\RR$ in Definition \ref{def3} such that $z\to m$ then, 
$$
{^L_aD_t^z}[\vf](t)\to \dfrac{d^m \vf(t)}{dt^m}\quad\text{and}\quad {^R_aD_t^z}[\vf](t)\to \dfrac{d^m \vf(t)}{dt^m},
\quad\forall t\in(a,b).
$$ 
Refer, for instance, to the monographs \cite{Kilbas, Podlubny} for detailed expositions.
\end{rem}
The following theorem holds.
\begin{thm}[{\cite[Property 2.1, Page 71]{Kilbas}}]\label{ThmDerivative}
If $\alpha,\beta\in\CC$ such that $\Re\{\alpha\}\geq 0$ and $\Re\{\beta\}>0$ then  
\begin{align}
{^L_aD_t^\alpha}[(t-a)^{\beta-1}](t) = \frac{\Gamma(\beta)}{\Gamma(\beta-\alpha)}(t-a)^{\beta-\alpha-1},\label{FracDerPowerL}
\\
{^R_tD_b^\alpha}[(b-t)^{\beta-1}](t) = \frac{\Gamma(\beta)}{\Gamma(\beta-\alpha)}(b-t)^{\beta-\alpha-1}, \label{FracDerPowerR}
\end{align}
for all $t\in(a,b) \subset \RR$. In particular,
\begin{align}
{^L_aD_t^\alpha}[1](t) = \frac{(t-a)^{-\alpha}}{\Gamma(1-\alpha)}
\quad\text{and}\quad 
{^R_aD_t^\alpha}[1](t) = \frac{(b-t)^{-\alpha}}{\Gamma(1-\alpha)}.
\end{align}
\end{thm}

\subsection{Adaptive Filters for System Identification}\label{ss:SIP}

The system identification problem is well-known in adaptive filtering. However, it is the main building block of the analytic arguments in this study. Therefore, we elaborate on it anyway to facilitate the ensuing discussion. Consider a simple problem of identification of a real-valued discrete-time filter with the unknown finite impulse response, 
$$
\bw_n:=\begin{bmatrix} w_n(0) & w_n(1) & \cdots & w_n(N-1)\end{bmatrix}^\top\in\RR^{N},
$$  
that describes the behavior of input, 
$$
\bx_n:=\begin{bmatrix} x_n & x_{n-1} & \cdots & x_{n-N+1}\end{bmatrix}^\top\in\RR^{N},
$$
to desired output, $d_n\in\RR$.  The subscript $n\in\mathbb{Z}$ is the time index and the superposed $\top$ indicates the transpose operation. The target output $d_n$ is furnished in order to supervise the filter weights $w_n$ so that  the system renders  filter output,  $y_n=\bw_n^\top\bx_n\in\RR$, that resembles the target in the least mean square sense. The mean square error (MSE), denoted by $\mathcal{E}$, is defined in terms of the output estimation error, $e_n:=d_n-y_n=d_n - \bw_n^\top \bx_n$ (for a fixed $\bw_n$),   as  
\begin{align}
\mathcal{E}(\bw_n) := E\left[e_n^2\right] = E\left[(d_n-\bw^\top_n \bx_n)^2\right]=\sigma_d^2-2\bw^\top_n\bp+\bw^\top_n\mathbf{R}\bw_n.\label{MSE}
\end{align}
Here, $\mathbf{R}:=E\left[\bx_n\bx_n^\top\right]\in\RR^{N\times N}$ is  the auto-correlation of  the input, $\bp:=E\left[d_n\bx_n\right]\in\RR^{N}$ is the cross-correlation between the input and the target output, and $\sigma_d^2:=E[d_n^2]$ is the variance of the target output. 

To solve the system identification problem,  the  MSE   is minimized for optimal weight filter, 
$$
\bw^\star:=\begin{bmatrix} w^\star(0) & w^\star(1)& \cdots & w^\star(N-1)\end{bmatrix}^\top. 
$$ 
Towards this end, the least mean squares (LMS) algorithm  is usually invoked. The idea is to allow the weights to be \emph{time-varying} so that they can be optimized in an iterative manner along the steepest descent of ${\mathcal{E}}$. Accordingly, utilizing the available data, the instantaneous gradient vector is derived as 
\begin{align}
\nabla_{\bw_n} \hat{\mathcal{E}}(\bw_n) = \nabla_{\bw_n} e_n^2 = \frac{\partial e^2_n}{\partial e_n}\frac{\partial e_n}{\partial y_n}\nabla_{\bw_n}y_n = -2e_n\bx_n,\label{ChainRule}
\end{align}
thanks to the chain rule for the ordinary derivatives of composite functions.
Since the negative of the gradient vector always points towards the direction of the steepest descent of the hyper-paraboloid surface formed by ${\mathcal{E}}$, directional increments opposite to the gradient vector gradually move the successive weight iterates closer to the minimum of ${\mathcal{E}}$. Accordingly, the LMS weight-update rule is defined as 
\begin{align}
\bw_{n+1} :=\bw_n-\frac{\mu_\ell}{2}\nabla_{\bw_n}\hat{\mathcal{E}}(\bw_n)= \bw_n +\mu_\ell e_n\bx_n,\label{LMSRule}
\end{align}
with an initial guess based on \emph{\`a priori} information and a parameter $\mu_\ell>0$ controlling the rate of learning. The LMS algorithm is very stable and efficient with a computational cost of $\mathcal{O}(N)$ per iteration. However, its convergence highly depends on the \emph{condition number} of the auto-correlation matrix $\mathbf{R}$ and is relatively poor.  Accordingly, various remedial interventions have been proposed in the conventional LMS algorithms and numerous modified algorithms are available achieving better convergence rates at the cost of increased computational complexity and reduced efficiency. A detailed account of these variants is out of the scope of the present investigation, however, it is emphasized that the ensuing discussion is also relevant to these variants and similar analyses hold with appropriate adjustments.

\subsection{Fractional Learning Algorithms}\label{ss:rules}

In the recent past, a plethora of fractional variants of the LMS and  steepest descent algorithms (SDA) have been proposed in adaptive signal processing wherein classical integer-order derivatives are fully or partially replaced by the fractional-order derivatives with order parameter (say) $\alpha\in (0,1)$. The simplest form of the fractional LMS iterate-update rule appears to be
\begin{align}
\bw_{n+1} : =\bw_n-\frac{\mu_\ell}{2} \nabla_{\bw_n} \left[\hat{\mathcal{E}}(\bw_n) \right] -\frac{\mu_f}{2}\, \nabla_{\bw_n}^\alpha\left[\hat{\mathcal{E}}(\bw_n) \right],
\label{FSDRule}
\end{align}
in terms of the  fractional gradient 
\begin{align*}
\left(\nabla^\alpha_{\bw_n} \hat{\mathcal{E}}(\bw_n)\right)_{l} := {^L_a}D_{w_n(l)}^\alpha\left[\hat{\mathcal{E}}(\bw_n)\right].
\end{align*}  
Here, $\mu_f$ is a control parameter that supervises the rate of learning due to fractional gradient. Some of the variant algorithms completely replace the integer-order gradient in conventional counterparts with that of fractional-order (in which case $\mu_\ell = 0$). On the other hand, some algorithms use both $\mu_\ell, \mu_f >0$, i.e., both fractional and integer-order gradients are assumed to play a role in the weight update through an iterative procedure. Moreover, when $\mu_f=0$, the fractional LMS algorithm tends to the LMS algorithm. 

There are several variants of the fractional LMS algorithms available in the literature. In this investigation, we discuss two representative iterate-update rules for brevity and clarity.  Nevertheless, we emphasize that the analysis performed here is relevant to other variants also since weight iterates corresponding to many of them usually follow slightly modified versions of the representative iterate-update rules.
The first representative iterate-update rule proposed in \cite{Raja1} suggests 
\begin{align}
w_{n+1}(l) =w_{n}(l) + \mu_{\ell}\, e_n\, x_{n-l} + \frac{\mu_f}{\Gamma(2-\alpha)} \,e_n x_{n-l} w_n^{1-\alpha}(l), \quad l=0,1,\cdots, N-1,
\label{FLMSRuleCom}
\end{align}
or in vector form 
\begin{align}
\bw_{n+1} =\bw_{n} + \mu_{\ell}\, e_n\,\bx_{n} + \frac{\mu_f}{\Gamma(2-\alpha)} \,e_n \bx_n\odot \bw_n^{1-\alpha}.
\label{FLMSRuleVec}
\end{align}
Here, $\odot$ represents element-wise product and notation $\bu^{1-\alpha}$ is used for element-wise  power of any $\bu\in\RR^N$, i.e.,   $\bu^{1-\alpha}= \begin{bmatrix} u^{1-\alpha}_1  & u_2^{1-\alpha}& \cdots & u_N^{1-\alpha}\end{bmatrix}$. 
Alternatively, by letting 
\begin{align}
\mathbf{F}_\alpha(\bu):=\frac{1}{\Gamma(2-\alpha)}{\rm diag}\left(u_1^{1-\alpha},u_2^{1-\alpha},\cdots, u_N^{1-\alpha}\right)\in\RR^{N\times N},
\end{align}
the iterate-update rule \eqref{FLMSRuleCom} can be written in vector form as 
\begin{align}
\bw_{n+1} =\bw_{n} + \mu_{\ell}\, e_n\,\bx_{n} + \mu_f \,e_n \mathbf{F}_\alpha(\bw_n)\bx_n.
\label{FLMSRuleF}
\end{align}
The second representative iterate-update rule proposed, e.g.,  in \cite{Cheng2017, Yin2019}, is given as
\begin{align}
w_{n+1}(l) =& w_n(l) +  \frac{\mu_f}{\Gamma(2-\alpha)}e_n x_{n-l}\left(w_n(l)-w_{n-1}(l)\right)^{1-\alpha}, \quad l=0,1,\cdots, N-1,
\label{FLMSRule2}
\end{align}
or in vector form as
\begin{align}
\bw_{n+1}(l) =& \bw_n(l) +  \mu_fe_n \mathbf{F}_\alpha(\bw_n-\bw_{n-1})  \bx_{n}.
\label{FLMSRule2Vec}
\end{align}

Remark that, there is a possibility of having negative values under fractional powers in \eqref{FLMSRuleCom} and \eqref{FLMSRule2} that could lead to complex outputs and could cause emergency exits. As a ready-witted remedy, absolute values are introduced in the iterate-update rules, i.e., 
\begin{align}
w_{n+1}(l) =&w_{n}(l) + \mu_{\ell}\, e_n\, x_{n-l} + \frac{\mu_f}{\Gamma(2-\alpha)} \,e_n x_{n-l}\Big|w_n(l)\Big|^{1-\alpha}, 
\label{FLMSRuleComMod}
\\
w_{n+1}(l) =& w_n(l) + \frac{\mu_f}{\Gamma(2-\alpha)}e_n x_{n-l}\Big|w_n(l)-w_{n-1}(l)\Big|^{1-\alpha}.
\label{FLMSRule2Mod}
\end{align}
The derivation of \eqref{FLMSRuleCom} and \eqref{FLMSRule2} from \eqref{FSDRule} will be discussed bit-by-bit in Section \ref{ss:derivations} and the underlying assumptions will be highlighted. 

\begin{rem}\label{Rem2}
It is already evident that the fractional LMS algorithms are computationally more expensive than the conventional LMS algorithms as they require additional operations to compute the fractional terms in the iterate-update rules. 
\end{rem}

\section{Schematic Kinks in Fractional Learning Algorithms}\label{s:Kinks}

In this section, we perform a rigorous mathematical analysis of the fractional LMS algorithms. We begin by deriving the representative iterate-update rules \eqref{FLMSRuleCom} and \eqref{FLMSRule2} with an aim to understand their main assumptions (Section \ref{ss:derivations}). Then, we investigate the origin and the remedies of complex outputs (Section \ref{ss:Complex}). We also discuss the geometrical interpretation of the fractional derivatives and thereby constructed learning algorithms (Section \ref{ss:geometry}). This will also help us understand the \emph{long memory} and \emph{short memory} characteristics. Finally, we discuss the convergence of these algorithms in stationary environments (Section \ref{ss:Convergence}).

\subsection{Derivation of Fractional Iterate-Update Rules}\label{ss:derivations}
We discuss the derivation of both representative iterate-update rules separately.

\subsubsection{Derivation of \eqref{FLMSRuleCom}}

Most common argument used while deriving the iterate-update rule \eqref{FLMSRuleCom} consists of using chain rule on objective functional $\hat{\mathcal{E}}(\bw_n)$ for evaluating its fractional gradient, exactly in the same fashion as in \eqref{ChainRule}; see, for instance, \cite{Raja1, Shah, Noise}. Precisely, the MSE in \eqref{MSE} is fractionally differentiated as 
\begin{align}
^L_0 D^\alpha_{w_n(l)} \left(\hat{\mathcal{E}}(\bw_n)\right) 
=& ^L_0 D^\alpha_{w_n(l)} \left(e_n^2\right) 
\nonumber
\\
=& \left(\frac{\partial e^2_n}{\partial e_n}\right)\left(\frac{\partial e_n}{\partial y_n}\right)\left(\frac{\partial y_n}{\partial w_n(l)}\right)\, ^L_0 D^\alpha_{w_n(l)} \left(w_n(l)\right)
\nonumber
\\
=&(2e_n)(-1)(x_{n-l})  \frac{w_n^{1-\alpha}}{\Gamma(2-\alpha)}
\nonumber
\\
=&-2e_n x_{n-l} \frac{w_n^{1-\alpha}}{\Gamma(2-\alpha)},\label{ChainRuleFrac}
\end{align}
using the formula \eqref{FracDerPowerL} with $a=0$. Eq. \eqref{ChainRuleFrac} renders iterate-update rules \eqref{FLMSRuleCom}, \eqref{FLMSRuleVec}, and \eqref{FLMSRuleF} on substitution in \eqref{FSDRule}. Unfortunately, the conventional chain rule used to derive \eqref{ChainRuleFrac} is mathematically invalid for fractional derivatives (see, for instance, \cite{Tarasov} for detailed discussion). The fractional chain rule for left Riemann-Liouville derivative is derived using \emph{Fa\`{a} di Bruno formula} and is given by (see, for instance, \cite[Eq. (2.209)]{Podlubny})
\begin{align}
^L_aD^p_t F(h(t)) =& \frac{(t-a)^{-p}}{\Gamma(1-p)}F(h(t))
+\sum_{k=1}^{\infty}\begin{pmatrix} p\\k \end{pmatrix}\frac{k!(t-a)^{k-p}}{\Gamma(k-p+1)}\sum_{m=1}^k F^{(m)}(h(t))\sum\ds\prod_{r=1}^k\frac{1}{a_r!}\left(\frac{h^{(r)}(t)}{r!}\right)^{a_r},
\end{align}
where $F$ and $h$ are sufficiently smooth functions, the sum $\sum$ extends over all combinations of non-negative integer values of $a_1,\cdots a_k$ such that 
\begin{align}
\sum_{r=1}^k ra_r =k\qquad\text{and}\qquad \sum_r^k a_r=m.
\end{align}

A mathematically valid procedure to derive the iterate-update rule \eqref{FLMSRuleCom} must avoid using fractional chain rule. Towards this end, the MSE in \eqref{MSE} is expanded as
\begin{align}
\hat{\mathcal{E}}(\bw_n)= \sigma_d^2-2\sum_{i=0}^{N-1} w_{n}(i)p_i(n)+\sum_{i,j=0}^{N-1} w_{n}(i)w_{n}(j) R_{ij}(n), 
\label{A1}
\end{align} 
where $p_i:=(\bp)_i$ and $R_{ij}:=\left(\mathbf{R}\right)_{ij}$ are the components of the cross-correlation vector $\bp$ and the auto-correlation matrix $\mathbf{R}$. To find the component fractional derivative, $ ^L_0 D^\alpha_{w_n(l)}[\hat{\mathcal{E}}]$, we re-arrange \eqref{A1} further as 
\begin{align*}
\hat{\mathcal{E}}(\bw_n)=& \sigma_d^2-2\sum_{\substack{i=0\\ i\neq l}}^{N-1} w_{n}(i)p_i(n)-2w_n(l) p_l(n)
+\sum_{\substack{i,j=0\\ i\neq l,j\neq l}}^{N-1}w_{n}(i)w_{n}(j) R_{ij}(n)
\\
&+\sum_{\substack{n=0\\ n\neq l}}^{N-1} w_{n}(i)w_n(l) R_{i l}(n)
+\sum_{\substack{j=0\\ j\neq l}}^{N-1} w_n(l)w_n(j) R_{l j}(n) 
+  w_n^2(l)R_{ll}(n)
\\
=& \Psi_n(l)
+2w_n(l)\left[\sum_{i=1, i\neq l}^{N-1} w_n(i) R_{i l}(n)-p_l(n)\right] +  w_n^2(l) R_{ll}(n),
\end{align*} 
where the fact that $\mathbf{R}$ is symmetric (i.e.,  $R_{ij}=R_{ji}$) is used. Here 
\begin{align}
\Psi_n(l):=\sigma_d^2-2\sum_{i=0, i\neq l}^{N-1} w_n(i)p_i(n)+\sum_{\substack{i,j=0\\ i\neq l, j \neq l}}^{N-1} w_n(i)w_n(j) R_{ij}(n),
\end{align} 
is a constant with respect to $w_n(l)$. Therefore, by the definition of the Left-Riemann-Liouville derivative and invoking the rule \eqref{FracDerPowerL}, one arrives at 
\begin{align}
^L_0D^\alpha_{w_n(l)}\left[\hat{\mathcal{E}}(\bw_n)\right]
=& 
\frac{w_n^{-\alpha}(l)}{\Gamma(1-\alpha)}\Psi_n(l)
+\frac{2w^{1-\alpha}_n(l)}{\Gamma(2-\alpha)}\left[\sum_{\substack{i=0\\ i\neq l}}^{N-1} w_n(i) R_{il}(n)-p_l(n)\right] 
\nonumber
\\
&+ \frac{2w^{2-\alpha}_n(l)}{\Gamma(3-\alpha)}R_{ll}(n).
\label{A3}
\end{align}

Let us now try to put \eqref{A3} in the form \eqref{ChainRuleFrac}. Towards this end,  we express the derivative as 
\begin{align}
^L_0D^\alpha_{w_n(l)}\left[\hat{\mathcal{E}}(\bw_n)\right]
= &
\frac{w_n^{-\alpha}(l)}{\Gamma(1-\alpha)}\Psi_n(l)
-\frac{2w^{1-\alpha}_n(l)}{\Gamma(2-\alpha)}\left[p_l(n)- \sum_{\substack{i=0}}^{N-1} w_n(i) R_{il}(n)+w_n(l)R_{ll}(n)\right] 
\nonumber
\\
& + \frac{2w^{2-\alpha}_n(l)}{\Gamma(3-\alpha)}R_{ll}(n) 
\nonumber
\\
=& 
\frac{w_n^{-\alpha}(l)}{\Gamma(1-\alpha)}\Psi_n(l)
-\frac{2w^{1-\alpha}_n(l)}{\Gamma(2-\alpha)}\left[p_l(n)- \sum_{\substack{i=0}}^{N-1} w_n(i) R_{il}(n)\right] 
\nonumber
\\ 
& + 2w^{2-\alpha}_n(l)R_{ll}(n)\left[\frac{1}{\Gamma(3-\alpha)}-\frac{1}{\Gamma(2-\alpha)}\right]
\nonumber
\\
=& 
\frac{w_n^{-\alpha}(l)}{\Gamma(1-\alpha)}\Psi_n(l)
-\frac{2w^{1-\alpha}_n(l)}{\Gamma(2-\alpha)}\Big[\left[d_n -\bw_n^\top\bx_n\right]x_{n-l}\Big]
-\frac{2(1-\alpha) w^{2-\alpha}_n(l)}{\Gamma(3-\alpha)}R_{ll}(n),
\end{align}
since $\Gamma(x+1)=x\Gamma(x)$, $p_l(n)=d_n x_{n-l}$, and $R_{il} =x_{n-i}x_{n-l}$. Therefore, one can conclude that 
\begin{align}
^L_0D^\alpha_{w_n(l)}\left[\hat{\mathcal{E}}(\bw_n)\right]
&\approx -2e_nx_{n-l}\frac{w^{1-\alpha}_n(l)}{\Gamma(2-\alpha)},
\end{align}
as in \eqref{ChainRuleFrac}, subject to following assumptions.

\begin{proof}[Assumptions]
\begin{enumerate}
\item[]
\item[A1] {\emph{For all values of $l$ and $n$, we have $w_n(l)>0$.}} 
\item[] This assumption was tacitly made when the Definition \ref{def3} of the left Riemann-Liouville derivative was used setting $a=0$ as the lower limit of the integral in \eqref{LRL}. Note that, $w_n(l)$ corresponds to the upper limit. This assumption may not be valid, especially, in the stochastic case. We will discuss this point in detail in Section \ref{sss:ComplexOutput}.
 
\item[A2] \emph{Additive constants in a function do not affect its extreme points.} 
\item[] This assertion stems from the hypothesis that the fractional derivative of the constant term, $\Psi_n(l)$, can be neglected without affecting the extrema of the function $\hat{\mathcal{E}}(\bw_n)$. This hypothesis is used in the literature (see, e.g., \cite[Remark 1]{Cheng2017}). The assertion is true when integer-order derivatives are used. However, it is invalid when fractional-order derivatives are used as they are non-zero for non-zero constants; see Theorem \ref{ThmDerivative}. We will elaborate on  this point further in Section \ref{sss:Extreme}. 

\item[A3] \emph{The first term dominates the expression 
\begin{align}
+\frac{2w^{1-\alpha}_n(l)}{\Gamma(2-\alpha)}\Big[\left[d_n -\bw_n^\top\bx_n\right]x_{n-l}\Big] +\frac{2(1-\alpha) w^{2-\alpha}_n(l)}{\Gamma(3-\alpha)}R_{ll}(n),
\end{align}
and therefore, the second term,
\begin{align}
\frac{2(1-\alpha) w^{2-\alpha}_n(l)}{\Gamma(3-\alpha)}R_{ll}(n),
\end{align}
can be suppressed.} 

\item[] This assumption may be reasonable when $\alpha\to1$. 


\end{enumerate}
\end{proof}

\subsubsection{Derivation of \eqref{FLMSRule2}}\label{ss:derivation2}

The iterate-update rule \eqref{FLMSRule2} has been derived, for instance, in \cite[Sec. 3.2]{Yin2019} and \cite[Sec. 3.2]{Cheng2017}. However, we briefly give the idea of the derivation for completeness. Towards this end, the MSE \eqref{MSE} is expanded as 
\begin{align}
\hat{\mathcal{E}}(\bw_n) 
=&\left(d_n-\sum_{i=0}^{N-1} w_n(i)x_{n-i}\right)^2
\nonumber
\\
=&\left(d_n-\sum_{\substack{i=0\\ i\neq l}}^{N-1} w_n(i)x_{n-i}-w_{n-1}(l)x_{n-l} - [w_n(l)-w_{n-1}(l)]x_{n-l}\right)^2
\nonumber
\\
=&\left(\Phi_n(l) - \Big[w_n(l)-w_{n-1}(l)\Big]x_{n-l}\right)^2
\nonumber
\\
=&\Phi_n^2(l) - 2\Phi_n(l)\Big[w_n(l)-w_{n-1}(l)\Big]x_{n-l}+\left(\Big[w_n(l)-w_{n-1}(l)\Big]x_{n-l}\right)^2,
\label{Expanded-E-2}
\end{align}
where
\begin{align}
\Phi_n(l) := d_n - \sum_{i=0, i\neq l}^{N-1} w_n(i) x_{n-i}-w_{n-1}(l)x_{n-l}.\label{Phi_n}
\end{align}
Differentiating \eqref{Expanded-E-2} using the left Riemann-Liouville  derivative  $_{w_{n-1}(l)}{^LD}^\alpha_{w_n(l)}$ defined through \eqref{LRL}, invoking the power rule \eqref{FracDerPowerL},  and  neglecting the constant term $\Phi_n^2(l)$ as in the previous case for the derivation of \eqref{FLMSRuleCom}, one arrives at
\begin{align}
_{w_{n-1}(l)}&{^LD}^\alpha_{w_n(l)}\left[\hat{\mathcal{E}}(\bw_n) \right]
\nonumber
\\
\approx
 &
-2\frac{\Phi_n(l)}{\Gamma(2-\alpha)}x_{n-l}\left[w_n(l)-w_{n-1}(l)\right]^{1-\alpha} 
+ \frac{2}{\Gamma(3-\alpha)}x^2_{n-l}\left[w_n(l)-w_{n-1}(l)\right]^{2-\alpha},
\nonumber
\\
\approx& -\frac{2}{\Gamma(2-\alpha)} e_n x_{n-l}\Big[w_n(l)-w_{n-1}(l)\Big]^{1-\alpha}.
\label{A4}
\end{align}
This furnishes  \eqref{FLMSRule2} on substitution in \eqref{FSDRule} together with $\mu_\ell=0$. Note that,  to arrive at equation  \eqref{A4}, following assumptions were  made.
\begin{proof}[Assumptions]
\begin{enumerate}

\item[]
\item[B1] \emph{For all values of $l$ and $n$, we have ${w_{n-1}(l)}<{w_n(l)}$.}
\item[]  As mentioned in the previous case, the first assumption tacitly made while using the Definition \ref{def3} of the left Riemann-Liouville derivative was ${w_{n-1}(l)}<{w_n(l)}$. Here, $w_{n-1}(l)$ and $w_n(l)$ respectively correspond to the lower and the upper limits of the integral in \eqref{LRL}. This assumption may not be valid, e.g., when any one of the optimal weights $w^\star(l)$ is negative. In that case, the sequence $(w_n(l))_{n\in\mathbb{N}}$ is expected to be decreasing so that $\ds\lim_{n\to\infty} w_n(l)$ converge to the negative value. We will discuss this point in detail in Section \ref{sss:ComplexOutput}.

\item[B2] \emph{Additive constants in a function do not affect its extreme points} 
\item[] This assertion stems from the hypothesis that the fractional derivative of the additive constant term, $\Phi^2_n(l)$, can be neglected without affecting the extrema of the function $\hat{\mathcal{E}}(\bw_n)$. We will elaborate on this point further in Sections \ref{sss:Extreme} and \ref{sss:short}. 

\item[B3] \emph{There exists a step size $\mu_f\in\RR$ such that $|w_n(l)-w_{n-1}(l)|<1$}. 
\item[] Under this assumption, in the expression
\begin{align*}
-2\frac{\Phi_n(l)}{\Gamma(2-\alpha)}x_{n-l}\left[w_n(l)-w_{n-1}(l)\right]^{1-\alpha} 
+ \frac{2}{\Gamma(3-\alpha)}x^2_{n-l}\left[w_n(l)-w_{n-1}(l)\right]^{2-\alpha},
\end{align*}
the first term is dominant and thus, the second term can be neglected. 
\end{enumerate}
\end{proof}
\begin{rem}\label{Rem3}
Under the assumption B3, if the step size $\mu_f$ is set properly, $w_n(l)$ changes slowly, i.e., $w_n(l)\approx w_{n-1}(l)$ and consequently, $\Phi_n(l)\approx e_n$.
Indeed, we have
\begin{align*}
\Phi_n(l) =& d_n - \sum_{i=0, i\neq l}^{N-1} w_n(i) x_{n-i}-w_{n-1}(l)x_{n-l}
\\
\approx & d_n - \sum_{i=0, i\neq l}^{N-1} w_n(i) x_{n-i}-w_{n}(l)x_{n-l}
\\
=& d_n - \sum_{i=0}^{N-1} w_n(i) x_{n-i}=e_n.
\end{align*}
The downside of Assumption (B3) is a slow convergence rate for the fractional LMS algorithms defined by the iterate-update rule of the type \eqref{FLMSRule2} (i.e., the algorithms with variable initial terms).
\end{rem}

\subsection{Emergence of Complex Outputs} \label{ss:Complex}

It is interesting to note that both representative iterate-update rules \eqref{FLMSRuleCom} and \eqref{FLMSRule2} (and, in fact, all the fractional iterate-update rules available in the literature) contain fractional powers of the quantities involving iterates $w_n(l)$. Therefore, whenever a fractional iterate $w_n(l)$ under the fractional power is negative, the resultant becomes complex.  It stymies the applicability of the fractional variants of the LMS algorithm for not only negative sought values but also for positive sought values. A simple justification is that the LMS iterate does not move in a straight path towards the optimal solution, it rather takes a \textit{zigzag} path (see, e.g., Fig. \ref{Fig.1}). That motivated a heuristic introduction of the absolute value in the update rules as given in \eqref{FLMSRuleComMod} and \eqref{FLMSRule2Mod} without any retrospective or prospective analysis. In this section, we identify the origin of the complex outputs and try to make sense of an absolute value in the iterate-update rules.

\begin{figure}[!htb]
\begin{center}
\includegraphics[width=0.5\textwidth, height=4cm]{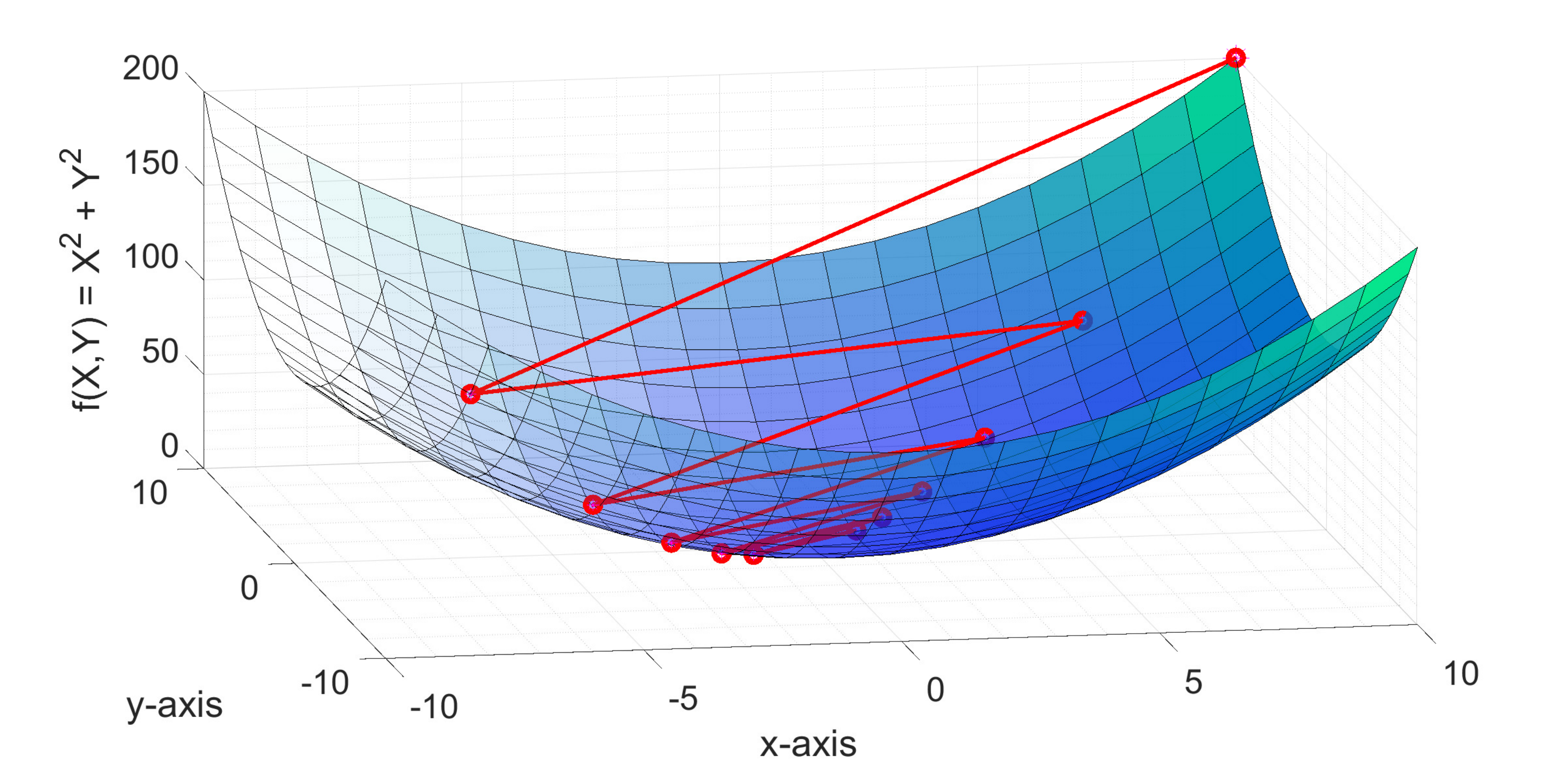}
\caption{Illustration of the LMS learning path.}\label{Fig.1}
\end{center}
\end{figure}

\subsubsection{The Genesis of Complex Outputs}\label{sss:ComplexOutput}

Let us start our discussion with the iterate-update rules of the form \eqref{FLMSRuleCom}. We recall that the left Riemann-Liouville fractional derivative $ ^L_0D^\alpha_{w_n(l)}$ was used with $a=0$ (the lower terminal limit in \eqref{LRL}). Accordingly, rule \eqref{FracDerPowerL} in Theorem \ref{ThmDerivative} for the fractional derivative of the power-law function was  invoked with the choice $a=0$ in the derivation of the iterate-update rule. This choice of $a$ correlates with the underlying (but unspecified) Assumption (A1) that 
\begin{align}
w_n(l)> 0, \qquad \text{for all}\quad 0\leq l\leq N-1,\quad n\in\mathbb{N}.
\end{align} 
Unfortunately,  it stymies the applicability of the iterate-update rule \eqref{FLMSRuleCom} for negative sought values and those scenarios of positive sought values when the path to the optimal solution of the LMS algorithm goes through negative iterates $w_n(l)$, i.e., when $w_n(l)<0$ for some  $l$  and $n$. Indeed, the fractional derivative \eqref{LRL} with $a=0$ and $t=w_n(l)$ will  produce a  complex denominator as $w_n(l)-\tau <0$ for fractional exponent $\alpha\in(0,1)$. At first, it seems strange when the fractional derivative of a real-valued function of a real variable defined over a real domain turns out to be complex. However, a closer look demystifies that the definition of the left Riemann-Liouville derivative \eqref{LRL} is used out of its domain of definition $\RR^+:=(0,+\infty)$. It is important to clarify here that there is no restriction \textit{\`a priori} on the choice of $a$; it could assume any real value (positive or negative) with appropriate consideration in the update rule. However, whatever the choice of $a$ is made, one must conform to the domain of definition of the derivative, i.e., $w_n(l)> a$ to avoid getting complex outputs. Thus, for $a=0$, the value $w_n(l)<0$ does not conform to the domain of definition and consequently, the algorithm furnishes absurd complex outputs.

A very similar observation can be made for the update rules of the type \eqref{FLMSRule2}. As we have specified in Section \ref{ss:derivation2}, the left Riemann-Liouville derivative $_{w_{n-1}(l)}{^LD}^\alpha_{w_n(l)}$, defined through \eqref{LRL}, was used to derive the iterate-update rule \eqref{FLMSRule2}. Therefore, the tacitly made (but unannounced) Assumption (B1) was that 
\begin{align}
w_{n-1}(l)< w_{n}(l),\qquad\text{for all}\quad 0\leq l\leq N-1,\quad n\in\mathbb{N},
\end{align}
that is, the sequence $(w_n(l))_{n\in\mathbb{N}}$ is pointwise monotone strictly  increasing for every $l$.  However, in case if the optimal weight $w^\star(l)$ was negative, and the algorithm was initialized by $0\leq w_0(l)< w_1(l)=1$, the sequence $(w_n(l))_{n\in\mathbb{N}}$ from \eqref{FLMSRule2} should be decreasing so that it could converge to $w^\star(l)<0$. On the other hand, if one assumes for the sake of argument that the algorithm is convergent, even then, a subsequence will be decreasing most likely since the LMS algorithm does not take a straight path towards the optimal solution, i.e., we may find ourselves in a situation where $w_{n-1}(l)> w_n(l)$ for some $l$ and $n$. This observation  is not restricted only to negative optimal weights. Indeed, if the positive weights are sought and we initialize the iterates with a positive value greater than the sought weight  for some $0\leq l \leq N-1$, we face the same issue. In all these situations, the algorithm will provide complex outputs. Towards this end,  whenever $w_n(l)<w_{n-1}(l)$, 
\begin{align}
_{w_{n-1}(l)}{^LD}^\alpha_{w_n(l)}[w_n(l)-w_{n-1}(l)]
=&\frac{1}{\Gamma(1-\alpha)}\frac{d}{dw_{n}(l)}\int_{w_{n-1}(l)}^{w_n(l)}\frac{\tau-w_{n-1}(l)}{(w_n(l)-\tau)^\alpha}d\tau
\nonumber
\\
=& 
\frac{1}{\Gamma(1-\alpha)}\frac{d}{dw_{n}(l)}\int_{w_n(l)}^{w_{n-1}(l)}\frac{w_{n-1}(l)-\tau}{(w_n(l)-\tau)^\alpha}d\tau.
\end{align}
This indicates that $\tau\in(w_n(l), w_{n-1}(l))$ and, therefore, the denominator is complex because $w_n(l)-\tau<0$. This can also be verified by the expression \eqref{FracDerPowerL} in Theorem \ref{ThmDerivative}. Indeed, we have 
\begin{align}
_{w_{n-1}(l)}{^LD}^\alpha_{w_n(l)}[w_n(l)-w_{n-1}(l)]
=&\frac{\Gamma(2)}{\Gamma(2-\alpha)}\Big(w_n(l)-w_{n-1}(l)\Big)^{1-\alpha}\in\CC.
\end{align}
We conclude once again that the left Riemann-Liouville derivative \eqref{LRL} is used outside of its domain of definition. In fact, the definition \eqref{LRL} was valid for $w_n(l) - w_{n-1}(l)\in\RR^+$ but it was used for $w_n(l) - w_{n-1}(l)\in\RR^-$. 

\subsubsection{Remedies and Justification of the Absolute Value}

We discuss the simple case of \eqref{FLMSRuleCom}. The ensuing discussion is also valid for the case of \eqref{FLMSRule2} with $w_n(l)$ replaced by $w_n(l)-w_{n-1}(1)$ and $^L_0D^\alpha_{w_n(l)}$ replaced by $_{w_{n-1}(l)}{^LD}^\alpha_{w_n(l)}$.

There are two possible ways to handle the problem of complex outputs. The first one is to identify a suitable lower limit $a$ such that
$$
\min_{l}w^\star(l)> a, \qquad 0\leq l\leq N-1,
$$
using \textit{\`a priori} information about the system. Unfortunately, it is not practical to identify such a suitable lower limit $a$ in the broader context of adaptive signal processing and general applications. The other way is to use the emergence of the complex outputs in the existing setup (with $a=0$) as an indication that \textit{one is going out of the domain of definition of the left-Riemann Liouville derivative}, i.e., 
\begin{align}
w_n(l)<0 \quad \text{whenever}\quad ^L_0D^\alpha_{w_n(l)}[e^2_n] \in\mathbb{C}\,\,\text{ or equivalently}\,\,\Im\left\{^L_0D^\alpha_{w_n(l)}[e^2_n]\right\}\neq 0.\label{Ccriterion}
\end{align}
Based on this criterion, one could take once again two possible routes discussed below.
\begin{enumerate} 
\item We first find an appropriate lower limit  $a<0$ so that 
\begin{align} 
\min_{l, n} w_n(l)> a, \qquad 0\leq l \leq N-1, \quad n\in\mathbb{N}.
\end{align}
Once such a number `$a$' is chosen, completely relaunch the entire algorithm by using the derivatives over the domain of definition interval $(a,w_n(l))$ in \eqref{LRL}.  Towards this end, one can set $a(n):=\ds\min_{l} w_n(l)$. However, this choice may not work for all $n$ because $\bw_n$ does not take a straight path from a time instance $n$ to another time instance $n+1$ to reach the optimal solution $\bw^\star$. Moreover, it may necessitate relaunching of the algorithm multiple times that will significantly increase computation cost as well as the efficiency of the algorithm. So, it is not practically suitable to adopt this remedy.

\item The second way is to make use of the \textit{right Riemann-Liouville derivative} \eqref{RRL} that has a  variable lower limit and a fixed upper limit at $b$, i.e., it has domain of definition $(t,b)$. When the criterion \eqref{Ccriterion}  indicates that $w_n(l)<0$, it will be suitable to apply right Riemann-Liouville derivative over the interval $(w_n(l), 0)$ because $w_n(l)$ will conform well to its domain of definition. The idea is to use the appropriate definition of the fractional derivative according to each input $w_n(l)$. The issue here is that the criterion \eqref{Ccriterion} needs to be verified on term-to-term bases, i.e., a decision has to be made for each and every $l$ and $n$, which is of course hectic. A quick fix to this is to identify what change will it bring to the update rule when we use right  derivative \eqref{RRL}  instead of the left one \eqref{LRL}. Towards this end, it can be easily seen that when \eqref{RRL} is used, the update rule \eqref{FLMSRuleCom} will read as
\begin{align}
w_{n+1}(l) =w_{n}(l) + \mu_{\ell}\, e_n\, x_{n-l} + \frac{\mu_f}{\Gamma(2-\alpha)} \,e_n x_{n-l} \Big(-w_n(l)\Big)^{1-\alpha}.
\label{FLMSRuleCom2}
\end{align}
Therefore, the fixed algorithm that switches the definition of the derivative term-by-term according to the criterion \eqref{Ccriterion} can be written by combining \eqref{FLMSRuleCom} and \eqref{FLMSRuleCom2} as 
\begin{align}
w_{n+1}(l) =w_{n}(l) + \mu_{\ell}\, e_n\, x_{n-l} + \frac{\mu_f}{\Gamma(2-\alpha)} \,e_n x_{n-l} 
\begin{cases}
\ds  \Big(w_n(l)\Big)^{1-\alpha}, & w_n(l)\geq 0,
\\
\ds  \Big(-w_n(l)\Big)^{1-\alpha}, & w_n(l) < 0,
\end{cases}
\end{align}
or equivalently 
\begin{align}
w_{n+1}(l) =
w_{n}(l) + \mu_{\ell}\, e_n\, x_{n-l} + \frac{\mu_f}{\Gamma(2-\alpha)} \,e_n x_{n-l} \left|w_n(l)\right|^{1-\alpha}.
\end{align}
\end{enumerate}
This justifies the use of an absolute value in the update rules \eqref{FLMSRuleComMod} and \eqref{FLMSRule2Mod}. Hence, Assumptions (A1) and (B1) are no more restrictive if the absolute values are used in the iterate-update rules. 

\subsection{Geometric Evaluation of the Fractional Learning Algorithms}\label{ss:geometry}
Let us now discuss the idea of using fractional derivatives in learning algorithms from a geometrical point of view.  It will also help us understand the relationship, if there is any, between the so-called \emph{fractional critical points} and the \emph{ordinary critical points} of a function.

\subsubsection{Geometrical Interpretation of  the Fractional Derivatives}

The geometrical interpretation of a fractional derivative is still unsettled and debatable despite being a centuries-old concept.  No generally acceptable geometric explanation has been provided yet since the appearance of the idea. Only a few vague interpretations have appeared so far that are far from being universally acceptable and practically functional. No solid connection is established in the literature between the fractional derivatives and the extreme or critical points of a sufficiently smooth function, unlike classical integer-order derivatives. It is well-known that the integer-order derivatives of a function $\varphi:\RR\to\RR$ are, specifically, local (pointwise defined) and are linked to the geometry of $\varphi$ and thus, have a clear geometrical meaning. Precisely, they provide suitable information about the behavior of the graph of $\varphi$, e.g., the regions where $\varphi$  is increasing, decreasing, concave, or the points where $\varphi$ has extreme values, inflections, cusps, vertical tangent, and so on. On the other hand, the fractional derivatives are non-local being defined in terms of an improper integral and have so-called \emph{memory} characteristics. Therefore, they provide very little punctual geometrical insight, at least regarding the behavior of the geometry of $\varphi$. We invite the interested readers to go through the articles \cite{Herrmann, Hilfer, Podlubny2, Tarasov2, Tavassoli} for detailed discussions regarding the geometrical and physical interpretations of the fractional derivatives. 

The conventional integer-order gradient vector has a geometrical and physical significance that has been vital in the success of the SDA. The gradient, $\nabla\varphi$ of a smooth function $\varphi$, defined by 
\begin{align}
\nabla_\bu  \varphi:= \begin{pmatrix}
\ds\frac{\partial \varphi}{\partial u_1} & \ds\frac{\partial \varphi}{\partial u_2} & \cdots & \ds \frac{\partial \varphi}{\partial u_{N}}
\end{pmatrix}^\top\in\RR^N, \quad\text{for any }\,\, \bu\in\RR^N,
\end{align}
points towards the direction in which $\varphi$ assumes its most pronounced increase in the slope and its length effectively renders the value of that slope (see, e.g., Fig. \ref{Fig2}). At each given point, the SDA learns the direction of the steepest descent of the function by means of the gradient vector.  In contrast, the fractional gradient, defined by 
\begin{align}
\nabla^\alpha_\bu  \varphi:= \begin{pmatrix}
{ ^L_a}D_{u_1}^\alpha[\varphi] & { ^L_a}D_{u_2}^\alpha[\varphi] & \cdots & {^L_a}D_{u_{N}}^\alpha\left[\varphi\right]
\end{pmatrix}^\top\in\RR^N, \quad\text{for any }\,\, \bu\in\RR^N,
\end{align}
points towards a direction other than that of the integer-order gradient. 
Therefore, it is impossible for it to point towards the direction of  the most pronounced increase in the slope of $\varphi$, unless exponent $\alpha\to1$, when  $\nabla^\alpha\varphi\to\nabla\varphi$ consequently (see, Remark \ref{Rem1}). Hence, the fractional SDA, learning the direction of the steepest descent through fractional gradients, cannot theoretically converge to an extreme point faster than the conventional SDA.  
\begin{figure}[!htb]
\begin{center}
\includegraphics[width=0.5\textwidth]{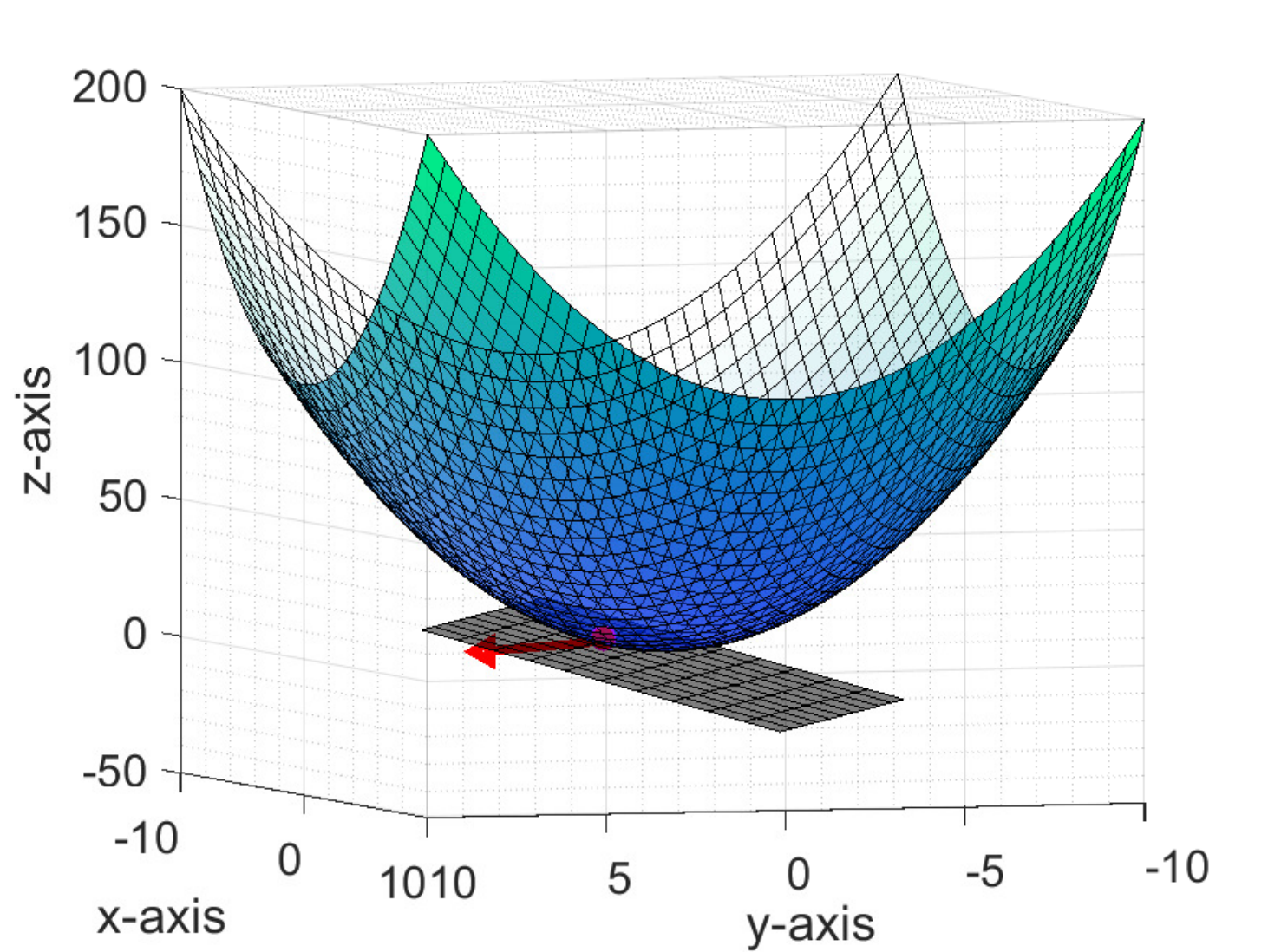}
\caption{Illustration of the gradient vector and the tangent plane at point $(1,2,5)$ of the function $z=x^2+y^2$.}\label{Fig2}
\end{center}
\end{figure}

\subsubsection{Fractional Extreme Points}\label{sss:Extreme}

Let us now explore the connection (if there is any) between the extreme points of a function with the so-called \emph{fractional extreme points}. We will also discuss the validity of Assumptions (A2) and (B2) about the constant terms and their influence on the fractional extreme points.

First of all, we note that the LMS iterate-update rule \eqref{LMSRule} renders an optimal solution $\bw^\star$ minimizing the sample MSE \eqref{MSE} when the instantaneous gradient vector $\nabla_{\bw_n}\hat{\mathcal{E}}(\bw_n)\to 0$. All such vectors $\bw_n\in\RR^N$ for which $\nabla_{\bw_n}\hat{\mathcal{E}}(\bw_n)=0$ are the critical points of the quadratic objective function $\hat{\mathcal{E}}(\bw_n) = e^2_n(\bw_n)$, where $e_n=d_n-\bw_n^\top \bx_n$ is the transmission error. Since, $\hat{\mathcal{E}}(\bw_n) = e^2_n(\bw_n)$ is a hyper-paraboloid, it has a unique critical point for which it achieves a minimum value. On contrary, the fractional rule \eqref{FSDRule} renders an optimal solution $\bw^\star_f$ that minimizes the sample MSE \eqref{MSE} when 
\begin{align}
\mu_\ell \nabla_{\bw_n}\hat{\mathcal{E}}(\bw_n)+\mu_f \nabla^\alpha_{\bw_n}\hat{\mathcal{E}}(\bw_n)= 0.
\label{frac-critical-cond}
\end{align}
We introduced the subscript $f$ in $\bw^\star_f$ in order to distinguish it from the classical optimal solution $\bw^\star$, where $f$ stands for \emph{fractional}. We will show in a while using a simple example (see, Example \ref{example}) that \eqref{frac-critical-cond} is not generally satisfied by the optimal solution $\bw^\star$. In other words, the fractional and conventional gradients have different critical points and thus, $\bw^\star$ and $\bw^\star_f$ are different.

For simplicity and clarity of our arguments in ensuing discussion,  we fix $\mu_\ell=0$ (i.e., only fractional gradient is used in fractional algorithms) and investigate the points where $\nabla^\alpha_{\bw_n}\hat{\mathcal{E}}(\bw_n)=0$. First, we consider the update rule \eqref{FLMSRuleCom} with $a=0$.  Evidently, it implies that 
\begin{align} 
0 = _0^LD^\alpha_{w_n(l)}[e^2_n(\bw_n)]=_0^LD^\alpha_{w_n(l)}[(d_n-\bw_n^\top\bx_n)^2],\qquad\text{for all} \quad 0\leq l\leq N-1.
\end{align}
Consider the one-dimensional case (i.e., $N=1$) for simplicity, and set $\bw_n=(w_n), \bx_n=(x_n)\in\RR^1$. Then, from Theorem \ref{ThmDerivative}, 
\begin{align}
{ ^L_0}D^\alpha_{w_n}\Big[(d_n-w_n x_n)^2\Big]
=&d_n^2 \Big({ ^L_0}D^\alpha_{w_n}[1]\Big)-2d_n x_n \Big({ ^L_0}D^\alpha_{w_n}[w_n]\Big) +  x_n^2\Big({ ^L_0}D^\alpha_{w_n}[w_n^2]\Big)
\nonumber 
\\
=& \frac{d_n^2}{\Gamma(1-\alpha)}w_n^{-\alpha}
 -\frac{2\Gamma(2)d_nx_n }{\Gamma(2-\alpha)}w_n^{1-\alpha}+\frac{x_n^2\Gamma(3) }{\Gamma(3-\alpha)}w_n^{2-\alpha}
\nonumber
\\
=& \frac{w_n^{-\alpha}}{\Gamma(3-\alpha)} \Big((2-\alpha)(1-\alpha)d_n^2
 -2(2-\alpha)d_nx_nw_n+2x_n^2 w_n^{2}\Big).
 \label{Mid1}
\end{align}
This furnishes a quadratic equation for critical points by setting ${ ^L_0}D^\alpha_{w_n}\Big[(d_n-w_n x_n)^2\Big]=0$ as 
\begin{align}
2x_n^2 (w_f^\star)^{2}
 -2(2-\alpha)d_nx_n (w_f^\star)+(2-\alpha)(1-\alpha)d_n^2=0.
 \label{QuadraticEq}
\end{align}
Hence, there are two critical points for the fractional gradient given by 
\begin{align}
w_f^\star=& \frac{(2-\alpha)d_nx_n\pm \sqrt{(2-\alpha)^2d_n^2x_n^2-2(1-\alpha)(2-\alpha)d_n^2x_n^2 }}{2x_n^2}
\nonumber
\\
=& 
 \frac{d_n}{2x_n}\Big((2-\alpha)\pm \sqrt{(2-\alpha)^2-2(1-\alpha)(2-\alpha)}\Big)
\nonumber
\\
=& 
 \frac{d_n}{2x_n}\Big((2-\alpha)\pm \sqrt{\alpha(2-\alpha)}\Big).
 \label{Roots}
\end{align}
It is worthwhile mentioning that the conventional integer-order gradient of the quadratic objective function $e_n^2$ has only  one critical point $w^\star=d_n/x_n$. 

\begin{rem}\label{Rem4}
Following remarks are in order. 
\begin{enumerate}
\item The fractional gradient always has two distinct critical points $w_f^\star$ for $\alpha\in(0,1)$ and both of them are different from the critical point $w^\star$ of the conventional gradient. Moreover, one of $w_f^\star$ approaches to  $w^\star=d_n/x_n$ while the other approaches to $0$ as $\alpha\to 1$.

\item It is not clear which one of $w_f^\star$ given in \eqref{Roots}  would the fractional LMS algorithm converge to in practice. It appears that the fractional algorithm is highly sensitive to the choice of the initial guess, and that will dictate the convergence of the algorithm to any one of those critical points. 

\item The roots $w^\star_f$ given in \eqref{Roots} are not only dependent on the objective function but also on the choice of the exponents, i.e., if we vary $\alpha$ then the fractional critical points of the objective function also vary accordingly.  Hence, the fractional exponent $\alpha$  induces significant deviations in the output of the fractional derivative by contributing to the steady-state error variably; instead of serving as an additional control parameter in the LMS algorithm for tailored convergence rate and enhancement of its performance.

\item Although $a=0$ is chosen above as the lower limit of the fractional integral, it can be seen readily that $a$ also influences the fractional critical points. Indeed, when $a\neq 0$, we will have fractional powers of $(w_n-a)$ instead of $w_n$ in the quadratic equation \eqref{QuadraticEq}. We elaborate further on this point through Example \ref{example}.

\item As for the influence of the constant terms on the critical points of the fractional gradients, note that while deriving \eqref{Mid1}, if we had removed the constant term $ \Big({ ^L_0}D^\alpha_{w_n}[d_n^2]\Big)$ as per Assumption (A2) then the quadratic equation \eqref{QuadraticEq} would have been 
\begin{align}
x_n^2w_n^2-(2-\alpha)d_nx_nw_n=0.
\label{QuadraticEq2}
\end{align}
This, in turn, would have furnished a set of two critical points $w_f^\star=0$ and $w_f^\star = (2-\alpha)d_n/x_n$. At one hand, it is clear that the constant terms have influence on the fractional critical point  $w^\star_f$, unlike $w^\star$ (see, Example \ref{example}). At the other hand, removing the derivatives of the additive constant terms can be a blessing in disguise as one of the roots of \eqref{QuadraticEq2} is always $w_f^\star= 0$ while the other one is  $w_f^\star=(2-\alpha)d_n/x_n =(2-\alpha)w^\star$ that is much similar to the true critical point of the conventional gradient, especially, when $\alpha\to 1$. 
\end{enumerate}
\end{rem}

We substantiate Remark \ref{Rem4} by a simple example below.
\begin{example} \label{example}
We choose a quadratic function $\varphi:[0,1] \to\RR$, defined by 
\begin{align}
\varphi(t):=2t^2- t+c, \qquad c\in\RR.
\label{phi}
\end{align}
It is trivially known that the only critical point of $\varphi$ is $t^\star=1/4$ irrespective of the choice of $c\in\RR$. Moreover, $\varphi(t)$ has a minimum at $t^\star$ since $\varphi''(1/4)>0$.  

We know that the true minimum point is $t^\star=1/4 >0$, so we can safely choose the lower limit of the fractional integral to be $0\leq a< 1/4$ (avoiding any breach of the domain of definition of fractional derivative) and look for fractional critical points $t^\star_f$ in the interval $(a,t)$. Our aim is to elaborate on the role of $a$, $c$, and $\alpha$ on $t^\star_f$. 

We first express $\varphi(t)$ as 
\begin{align*}
\varphi(t)=& 2((t-a)+a)^2-((t-a)+a)+ c  
= 2(t-a)^2+(4a-1)(t-a)+(2a^2-a+c),
\end{align*}
so that its  left Riemann-Liouville derivative of order $\alpha\in(0,1)$ over $(a,t)$ is 
\begin{align}
^L_aD^\alpha_t(\varphi(t)):=& 
\frac{4}{\Gamma(3-\alpha)}(t-a)^{2-\alpha}+\frac{(4a-1)}{\Gamma(2-\alpha)}(t-a)^{1-\alpha}+\frac{(2a^2-a+c)}{\Gamma(1-\alpha)}(t-a)^{-\alpha}
\nonumber
\\
=& \frac{(t-a)^{-\alpha}}{\Gamma(3-\alpha)}\Big(4(t-a)^2+(4a-1)(2-\alpha)(t-a)+(2a^2-a+c)(1-\alpha)(2-\alpha)\Big).
\end{align}
Therefore, the fractional critical points satisfy the equation
\begin{align}
4(t^\star_f-a)^2+(4a-1)(2-\alpha)(t^\star_f-a)+(2a^2-a+c)(1-\alpha)(2-\alpha)=0,
\end{align}
and are given by 
\begin{align}
t_f^\star = a+\frac{-(4a-1)(2-\alpha)\pm \sqrt{(4a-1)^2(2-\alpha)^2-16(2a^2-a+c)(1-\alpha)(2-\alpha)}}{8}.
\end{align}
After fairly easy manipulations, we get 
\begin{align}
t_f^\star = a+\frac{-(4a-1)(2-\alpha)\pm\sqrt{(2-\alpha)\Big[\alpha(4a-1)^2-2(8c-1)(1-\alpha)\Big]}}{8}.
\end{align}

One can draw following conclusions.

\begin{enumerate}
\item The fractional critical points depend on $\alpha, a$, and $c$. 
\item The points $t^\star_f$ are imaginary when 
\begin{align}
c>\frac{2(1-\alpha)+\alpha(4a-1)^2}{16(1-\alpha)},\label{cond-c}
\end{align}
i.e., there are no real points such that $ ^L_aD_t^\alpha\varphi(t) =0$. We remind here that additive constant $c$ had no role in the extreme values of the integer-order derivative. It substantiates our point that the additive constants have a significant influence on the optimal solution in the fractional LMS algorithm. 
Moreover, quadratic equation \eqref{QuadraticEq2} is a particular case always having two real distinct roots for all $\alpha\in (0,1)$ thanks to a specific choice of $c$ corresponding to the MSE objective functional. 

\item Both the roots $t^\star_f$ are different from $t^\star=1/4$. 

\item Finally, it is interesting that, for $a=0$,
\begin{align}
^L_0D_t^\alpha[\varphi(t^\star)] =& \frac{(t^\star)^{-\alpha}}{\Gamma(3-\alpha)}\Big(4(t^\star)^2-(2-\alpha)t^\star+c(1-\alpha)(2-\alpha)\Big)
\nonumber
\\
=&\frac{4^\alpha}{\Gamma(1-\alpha)}\left(c-\frac{1}{4(1-\alpha)}\right),
\label{eq:37}
\end{align}
which is strictly negative at the true critical point $t^\star=1/4$ for all  $c<{1}/{4(1-\alpha)}$. Note that, there are two real values of $t^\star_f$ for such $c$ since it  avoids the condition \eqref{cond-c}. This justifies our claim that \eqref{frac-critical-cond} does not hold in general at the true critical point. More precisely, the ordinary derivative at $t^\star$ is zero but the fractional derivative at $t^\star$ in \eqref{eq:37} is non-zero (except for $c=1/4(1-\alpha)$), and therefore, the sum in \eqref{frac-critical-cond} is also non-zero at $t^\star$. 
\end{enumerate}
\end{example}

We end this subsection with the following general result regarding the behavior of fractional derivatives at the true minimum that justifies the observation made in \eqref{eq:37}. 
\begin{thm}[{\cite[Theorem 2.4]{Refai}}]\label{Thm2}
If $\varphi:[0,1]\to \RR$ is a twice continuously differentiable function that attains its minimum at $t^\star\in(0,1)$ then
\begin{align}
^L_0D_t^\alpha[\varphi(t^\star)]\leq \frac{(t^\star)^{-\alpha}}{\Gamma(1-\alpha)}\varphi(t^\star), \qquad \alpha\in(0,1).
\end{align}
\end{thm}

\subsubsection{Short-Memory Characteristic}\label{sss:short}

It has been mentioned earlier that the fractional derivatives are non-local because they are defined in terms of integrals. Therefore, they have a so-called \emph{long-memory} characteristic, i.e., they carry forward  information of all the past states. As a way to deal with this non-locality issue, iterate-update rules of the form \eqref{FLMSRule2} were designed by iterating the lower limit of the fractional derivative \eqref{LRL} for each time instance $n$ so that only a short memory of the fractional derivative could be retained (the so-called \emph{short-memory characteristic}). A natural choice of the variable limit was $a=a_{n,l}=w_{n-1}(l)$. Therefore, ${_{w_{n-1}(l)}}{^LD}^\alpha_{w_n(l)}\left[\hat{\mathcal{E}}(\bw_n) \right]$ was used in fractional LMS iterate-update rule \eqref{FLMSRule2}. Below, we discuss the critical values corresponding to this fractional derivative to elaborate on the role of this short-term memory effect on the performance of the fractional LMS algorithms. 

As in the Section \ref{sss:Extreme},  we consider the one-dimensional case for simplicity. Using the form \eqref{Expanded-E-2} of the sample MSE, we evaluate the fractional derivative as 
\begin{align}
{_{w_{n-1}}}&{^LD}^\alpha_{w_n}\Big[(d_n-w_n x_n)^2\Big]
\nonumber
\\
=&\Phi_n^2(0) \Big({_{w_{n-1}}}{^LD}^\alpha_{w_n}[1]\Big)-2\Phi_n(0) x_n \Big({_{w_{n-1}}}{^LD}^\alpha_{w_n}[\Delta w_n]\Big) 
+  x_n^2\Big({_{w_{n-1}}}{^LD}^\alpha_{w_n}[(\Delta w_n)^2]\Big)
\nonumber 
\\
=& \frac{\Phi_n^2(0)}{\Gamma(1-\alpha)}\left[\Delta w_n\right]^{-\alpha}
 -\frac{2\Phi_n(0)x_n }{\Gamma(2-\alpha)}\left[\Delta w_n\right]^{1-\alpha}+\frac{2x_n^2}{\Gamma(3-\alpha)}\left[\Delta w_n\right]^{2-\alpha}
\nonumber
\\
=& \frac{\left[\Delta w_n\right]^{-\alpha}}{\Gamma(3-\alpha)} \Bigg((2-\alpha)(1-\alpha)\Phi_n^2(0)
-2(2-\alpha)\Phi_n(0)x_n\left[\Delta w_n\right]+2x_n^2 \left[\Delta w_n\right]^{2}\Bigg),
 \label{Mid2}
\end{align}
where $\Phi_n(l)$ is given in \eqref{Phi_n} and $\Delta w_n:=(w_n-w_{n-1})$.  This furnishes a quadratic equation for critical points $w^\star_f$ by setting the fractional derivative in \eqref{Mid2} to zero, i.e.,
\begin{align}
2x_n^2 (w_f^\star -w_{n-1})^{2} -2(2-\alpha)\Phi_n(0)x_n (w_f^\star -w_{n-1})+(2-\alpha)(1-\alpha)\Phi_n^2(0)=0.
 \label{QuadraticEq3}
\end{align}
Remark that the quadratic equation \eqref{QuadraticEq3} is essentially equivalent to  \eqref{QuadraticEq} with $w_f^\star$ and $d_n$ replaced with $w_f^\star -w_{n-1}$ and $\Phi_n^2(0)$, respectively.  
Hence, there are two critical points for the fractional gradient given by 
\begin{align}
w_{f,\pm,n}^\star=& w_{n-1}+
 \frac{d_n-w_{n-1}x_n}{2x_n}\Big((2-\alpha)\pm \sqrt{\alpha(2-\alpha)}\Big),\label{Roots2}
\end{align}
and the conclusions drawn in Remark \ref{Rem4} in Section \ref{sss:Extreme} are also relevant to the algorithms designed with short memory characteristics. However, now there are two sequences of critical points $(w_{f,\pm, n}^\star)$ (as the fractional derivative changes each time due to the variable lower limit $a_{n,l}=w_{n-1}(l)$). 

There is no apparent big difference between the algorithm derived using short memory characteristic due to variable initial terms and the one without it. However, this is not true. The difference lies with the convergence guarantee and the steady state error produced by the two variants. As will be discussed in Section \ref{ss:Convergence}, the fractional part of the algorithm \eqref{FLMSRuleComMod} approaches zero as $n\to \infty$ and the convergence of the algorithm is achieved thanks to the integer-order gradient part. The fractional part merely contributes to the steady state error as it seldom becomes zero at the true extreme point in finite time (see Observation $4$ in Example \ref{example} and Theorem \ref{Thm2}). On contrary, the short memory algorithms of the form \eqref{FLMSRule2Mod} usually do not have integer gradient part (i.e., $\mu_\ell=0$), but they have guaranteed convergence to the Wiener solution in stationary cases under the assumption of the convergence of the LMS solution thanks to the short memory characteristic; see, Proposition \ref{Prop2}. However, their rate of convergence is slower than the standard LMS algorithm due to their construction under Assumption (B2); see Remark \ref{Rem3} and Example \ref{example2}. Moreover, the critical points corresponding to the rule \eqref{FLMSRule2Mod}, $w^\star_{f,\pm,n}$, both approach to $w^\star$ as $n\to \infty$ for any $\alpha\in(0,1)$, unlike those corresponding to \eqref{FLMSRuleComMod}.


\subsection{Convergence Analysis}\label{ss:Convergence}
Let us first discuss the convergence of the algorithm \eqref{FLMSRule2Mod}. 

\begin{prop}\label{Prop2}
Under the assumption that the sequence $(\bw_n)_{n\in \mathbb{N}}$ given by the update rule \eqref{LMSRule} converges to the Wiener solution $\bw^\star$, the sequence $(\bw^\star_{f,\pm,n})_{n\in\mathbb{N}}$ also converges to $\bw^\star$. 
\end{prop}
\begin{proof}
We prove the result for one-dimensional case. We refer to \cite{chen2017study} for more detailed discussions on the general case. Let $\ds\lim_{n\to\infty}w_n=w^\star$. Then, it can be seen that the sequence $(w^\star_{f,\pm,n})_{n\in\mathbb{N}}$ also converges to $w^\star$. Indeed, from \eqref{Roots2}, 
\begin{align}
\ds\lim_{n\to\infty} w^{\star}_{f,\pm,n} = \lim_{n\to\infty}w_{n-1}+\ds\lim_{n\to\infty}\frac{d_n-w_{n-1}x_n}{2x_n}\Big((2-\alpha)\pm\sqrt{\alpha(2-\alpha)}\Big) = w^\star.
\end{align}
\end{proof}
As we mentioned in the previous section, we substantiate our claim that the convergence of the algorithm \eqref{FLMSRule2Mod} is slower than the standard LMS algorithm through a simple example below. 
\begin{example}\label{example2}
We choose a quadratic function $\varphi_2:[0,1] \to\RR$, defined by 
\begin{align}
\varphi_2(t):=(2t-3)^2, 
\label{phi2}
\end{align}
that has the only critical point $t^\star=3/2$ where it is minimum. In Fig. \ref{Fig3}, we show the performance of the fractional algorithm
\begin{align}
t_{n+1}=t_n-\frac{\mu_f}{2}\, {_{t_{n-1}}}{^LD}^\alpha_{t_n}\left[\varphi_2(t_{n})\right]. 
\end{align}
It can be observed that the solution converges to the true extreme point $t^\star=3/2$. However, the convergence is very slow as anticipated in Remark \ref{Rem3}. 

\begin{figure}[!htb]
\begin{center}
\includegraphics[width=0.4\textwidth]{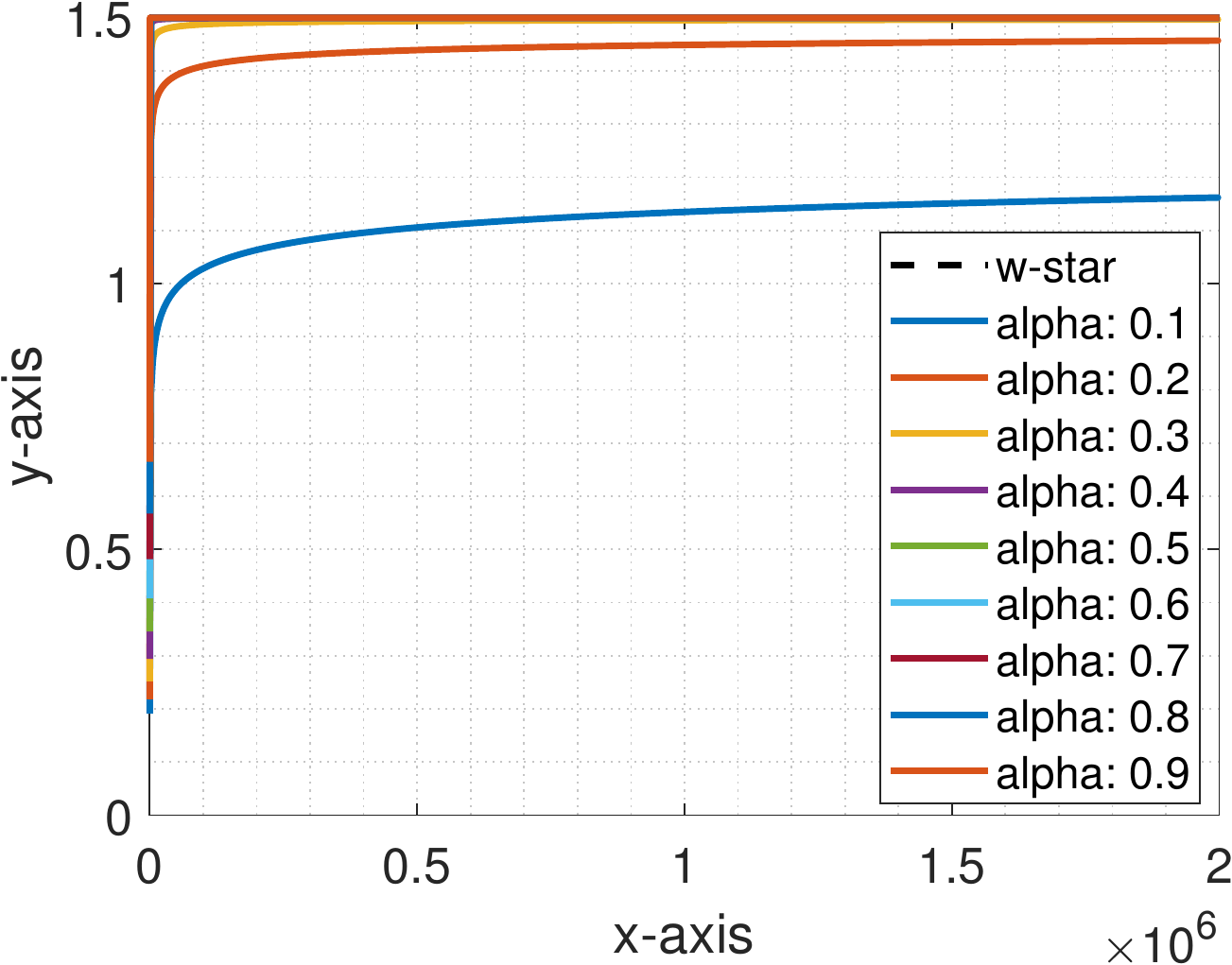}
\hspace{0.05\textwidth}
\includegraphics[width=0.4\textwidth]{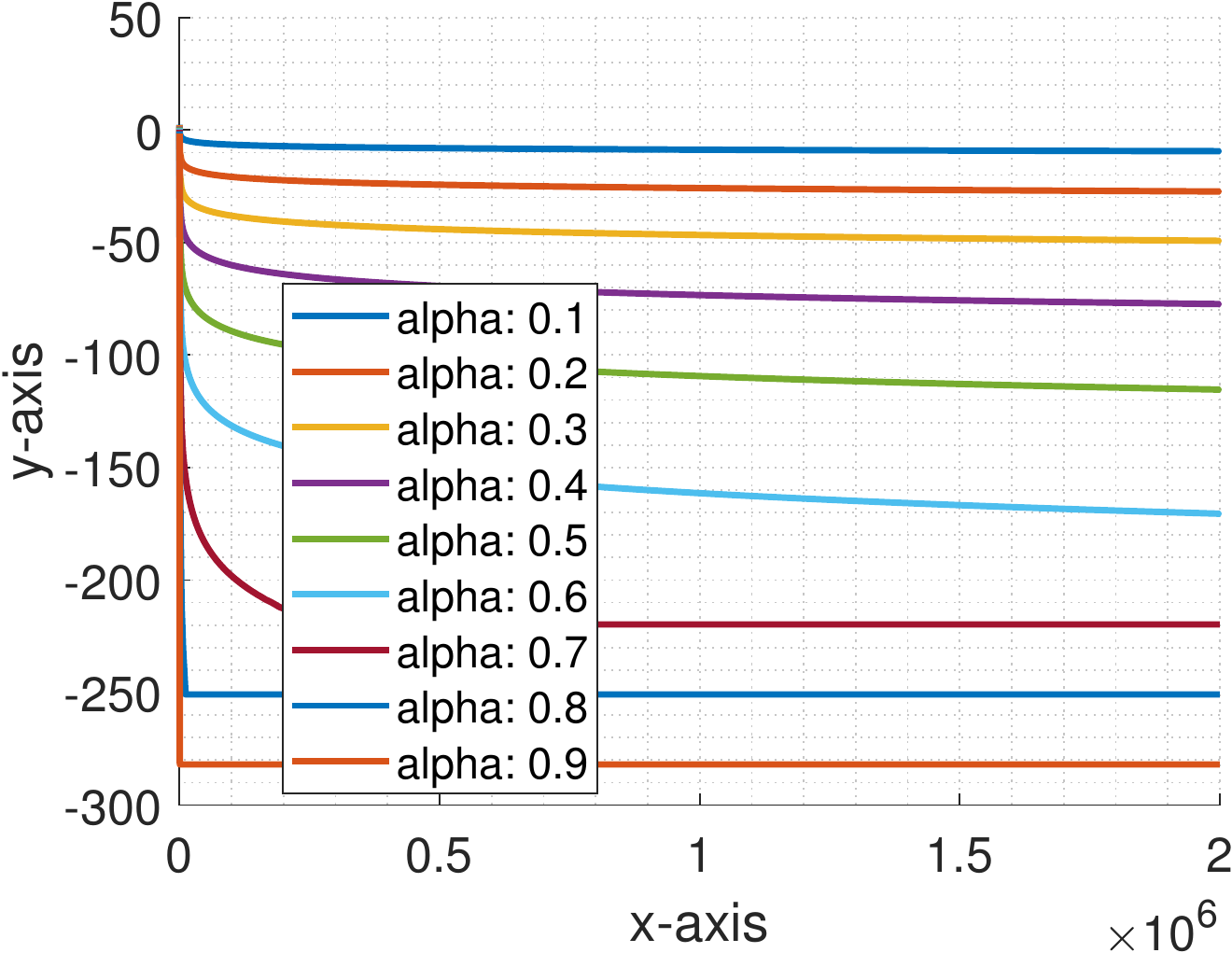}
\caption{Illustration of Example \ref{example2}. Left: Optimal Solution. Right: Learning Curves.}\label{Fig3}
\end{center}
\end{figure}
\end{example}

Let us now discuss the convergence of the algorithm \eqref{FLMSRuleComMod}. Towards this end, we have the following lemma. 
\begin{lem}\label{Prop1}
If the sequence $(\bw_{n})_{n\in\NN}$ given by the update rule \eqref{LMSRule} converges to the Wiener solution $\bw^\star$ then the sequence $(\bu_n)_{n\in\NN}\subset \RR^N$, defined by
\begin{align}
u_n(l):=\mu_f e_nx_{n-l}|w_n(l)|^{1-\alpha}, \qquad 0\leq l\leq N-1,
\end{align} 
is also convergent and $\bu_n\to \mathbf{0}\in\RR^N$.
\end{lem}
\begin{proof}
If $\bw_{n}$ converge to $\bw^\star$ as $n\to+\infty$, then $w_n(l)\to w^\star(l)$ for each $l$. Recall that every convergent sequence is bounded, therefore, there exist a number $M\in\RR^+$ such that $|w_n(l)|\leq M$ for all $0\leq l\leq N-1$ and $n\in\NN$. Let $\epsilon >0$ be arbitrary. By convergence of $w_n(l)\to w^\star(l)$, there exists a natural number $N_\epsilon$ such that
\begin{align}
|w_k(l)-w^\star(l)| <\epsilon \left( \frac{\mu_\ell}{2\mu_f M^{1-\alpha}}\right), \qquad \text{ for all }\,\, k > N_\epsilon+1. \label{Mid3}
\end{align}
Note also that, by definition of the general term given in \eqref{LMSRule}, we have \begin{align}
|\mu_\ell e_k x_{k-\ell}| = & | w_k(l)-w_{k-1}(l)| 
\nonumber 
\\
\leq & | w_k(l)-w^\star (l)|+| w_{k-1}(l)-w^\star(l)|
\nonumber 
\\ 
<&\epsilon \left( \frac{\mu_\ell }{\mu_f M^{1-\alpha}}\right), 
\qquad\qquad\qquad\qquad\qquad
\text{for all} \quad 
k > N_\epsilon.
\label{Mid4}
\end{align}
Thus, from \eqref{Mid3} and \eqref{Mid4},  
\begin{align}
|u_k(l)| = |\mu_f e_k x_{k-l}||w_k|^{1-\alpha}\leq \frac{\mu_f}{\mu_\ell}  | \mu_\ell e_k x_{k-l}| M^{1-\alpha} < \epsilon, \qquad
\text{for all} \quad 
k > N_\epsilon.
\end{align}
This completes the proof. 
\end{proof}
\begin{thm}\label{Thm3}
Under the assumption of Lemma \ref{Prop1}, the sequence $(\bv_{n})_{n\in\mathbb{N}}$ given by $\bv_n:=\bw_n+\bu_n$, generated by algorithm  \eqref{FLMSRuleComMod}, converges to $\bw^\star$. 
\end{thm}
\begin{proof}
Simply consider $(\bv_n)_{n\in\NN}$ by $v_n(l):=w_n(l)+u_n(l)$ where $w_n(l)$ and $u_n(l)$ are as in Lemma \ref{Prop1}. Then 
\begin{align}
\lim_{n\to\infty}\bv_n =\lim_{n\to \infty}\bw_n +\lim_{n\to\infty}\bu_n = \bw^\star+\mathbf{0} = \bw^{\star}.
\end{align}
\end{proof}

Theorem \ref{Thm3} substantiates that the convergence of the algorithm \eqref{FLMSRuleComMod} is achieved due to the integer-order gradient factor. For finite number of iterations, the fractional term is not going to vanish no matter how small it may be and, therefore, contribute to the steady state error. In the next section, we conduct a few numerical experiments to support our findings.  

\section{Numerical Simulations and Discussions}\label{s:Num}

To compare the performance of the fractional and integer-order gradient-based LMS algorithms, we consider three evaluation protocols:  
\begin{itemize}
\item[(i)] a system with negative desired weights, 
$$
\bw=
\begin{bmatrix}
-15, &-14, &\cdots, &-2, &-1
\end{bmatrix},
$$  
under both noise-free and noisy environments;  
\item[(ii)] a system with positive desired weights,
$$
\bw=
\begin{bmatrix}
1, & 2, &\cdots, & 14, & 15
\end{bmatrix},
$$  
under both noise-free and noisy environments;  
\item[(iii)] a system with random desired weights under noisy environment using modified weight-update rules \eqref{FLMSRuleComMod} and \eqref{FLMSRule2Mod}.   
\end{itemize} 

\subsection{Experimental Setup}

In the rest of Section \ref{s:Num},  we consider the noisy environments with signal-to-noise ratio (SNR) of $10$dB. The LMS and its fractional-order variants are configured to equal performance at $\alpha=1$.  The performance of the fractional LMS algorithms is observed for fractional exponents $\alpha=0.9$, $0.8$, $0.7$, $0.6$, $0.5$, and $0.4$. For all experiments, the step-size of LMS was fixed as $\mu_{l} = 1\times 10^{-2}$, the step-sizes of fractional LMS algorithm \eqref{FLMSRuleComMod} were fixed as $\mu_{l} = 5\times 10^{-3}$ and $\mu_{f} = 5\times 10^{-3}$, and the step-size of fractional LMS algorithm \eqref{FLMSRule2Mod} was fixed as $\mu_{f} = 1\times 10^{-2}$.  

Protocols (i) and (ii) are used to substantiate that the iterate-update rule \eqref{FLMSRuleCom} (without modulus) is affected by complex outputs whether positive or negative weights are sought and, consequently, the fractional algorithm is  divergent. Similar experimental trends can be delineated for the second iterate-update rule \eqref{FLMSRule2}. Protocol (iii) is used for modified iterate-update rules \eqref{FLMSRuleComMod} and \eqref{FLMSRule2Mod} (with modulus).
For system input, we have considered a random signal of length $1000$ obtained from a zero-mean Gaussian distribution with unit variance. The experiments are repeated for $1000$ independent rounds and mean results are reported.  For each independent round, the weights were initialized with zeros except in algorithm \eqref{FLMSRule2Mod} where weights were initialized randomly from a Gaussian distribution with unit variance.  The performance of all the algorithms is evaluated on mean deviation (MD) which is the $\ell_1-$ norm of the difference between the sought and the approximated weights, i.e., 
\begin{equation*}
	\Delta \bw_n= \frac{\|\bw_n- \hat{\bw}_n\|_{\ell_1}}{N},
\end{equation*}
where $\bw_n$ and $\hat{\bw}_n$ are the sought and approximated weight vectors at $n$th iteration, respectively. Here, $\|\cdot\|_{\ell_1}=|\cdot|$ is the $\ell_1-$ norm and $N$ is the length of the filter vector.

\subsection{Complex Outputs and Divergence}

Figures \ref{FLMS_neg} and \ref{FLMS_pos} show the learning curves for the LMS algorithm and the fractional LMS algorithm \eqref{FLMSRuleCom} for evaluation protocols (i) and (ii), respectively .  For both protocols, we set up both algorithms on an equal convergence rate and compare their steady-state performance.  It can be observed that the fractional LMS algorithm \eqref{FLMSRuleCom} failed to identify the system with negative (Fig. \ref{FLMS_neg}) as well as positive weights (Fig. \ref{FLMS_pos}) for all the listed values of $\alpha$. This conforms to our theoretical findings in Section \ref{ss:Complex}. Similar results hold for the fractional LMS algorithm with iterate-update rule \eqref{FLMSRule2}. 
\begin{figure}[!htb]
\begin{center}
\includegraphics[width=0.45\textwidth]{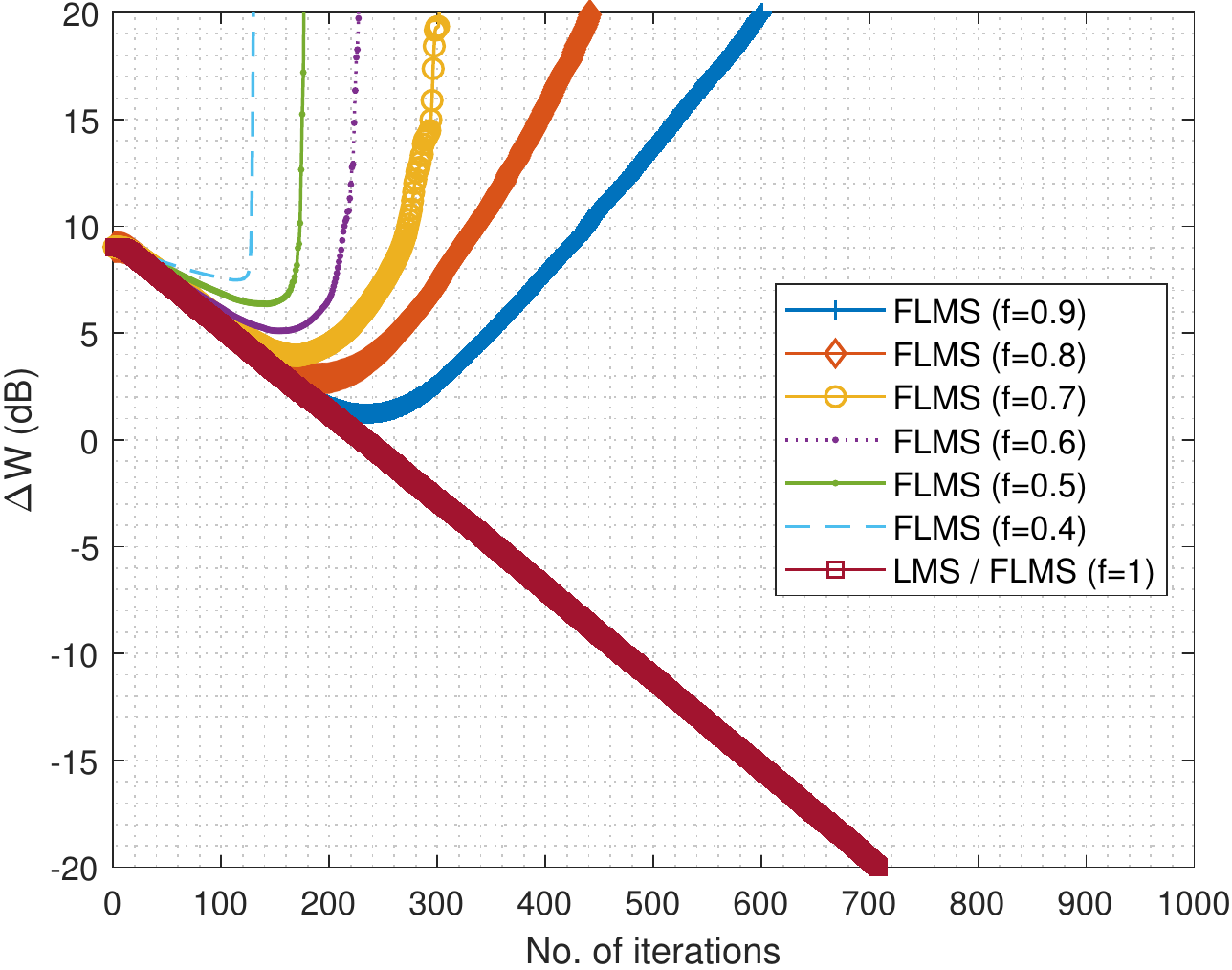}
\hspace{0.05\textwidth}
\includegraphics[width=0.45\textwidth]{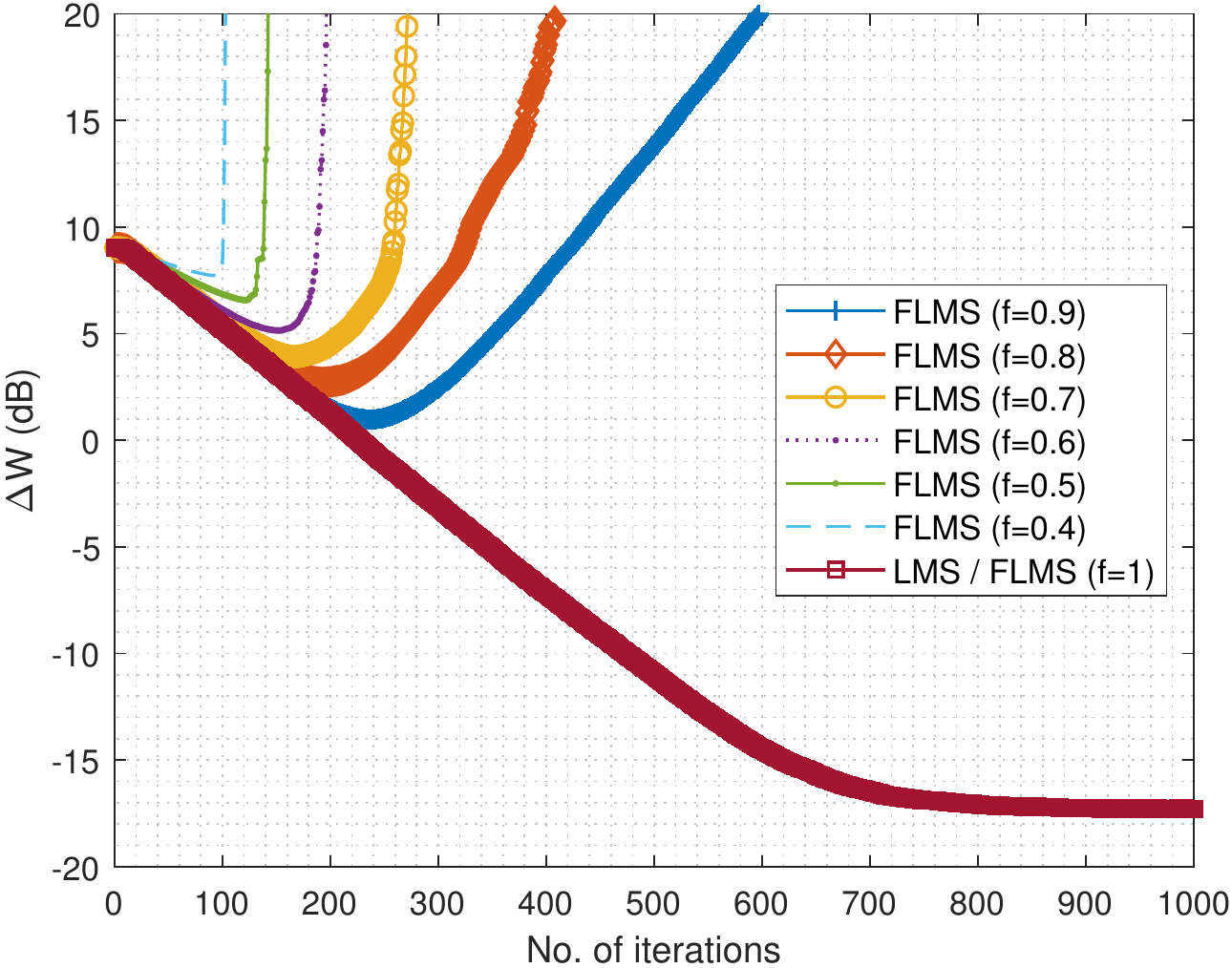}
\caption{Protocol (i): Learning curves for different values of fractional power for  noise-free (left) and noisy (right) environments.}\label{FLMS_neg}
\end{center}
\end{figure}
\begin{figure}[!htb]
\begin{center}
\includegraphics[width=0.45\textwidth]{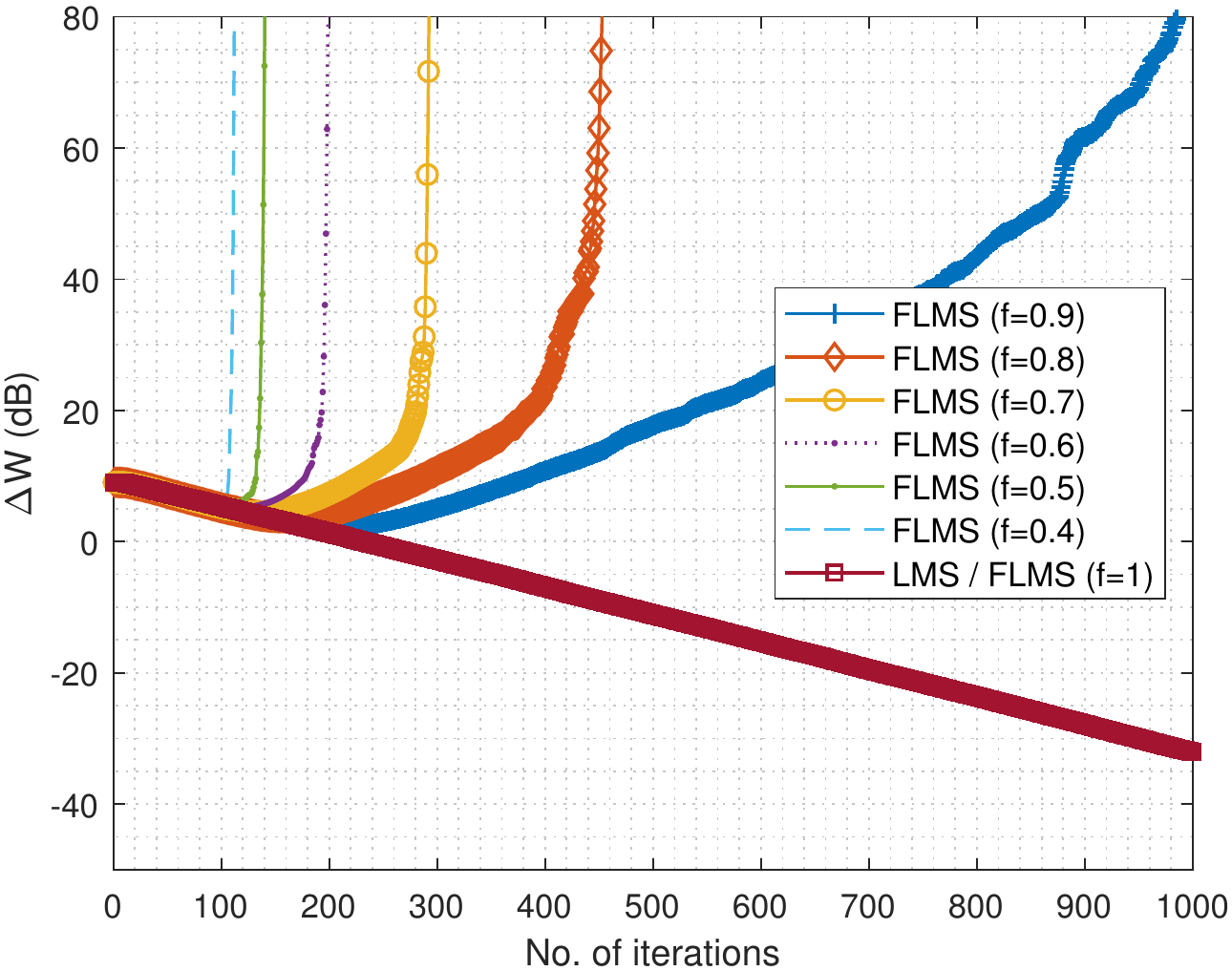}
\hspace{0.05\textwidth}
\includegraphics[width=0.45\textwidth]{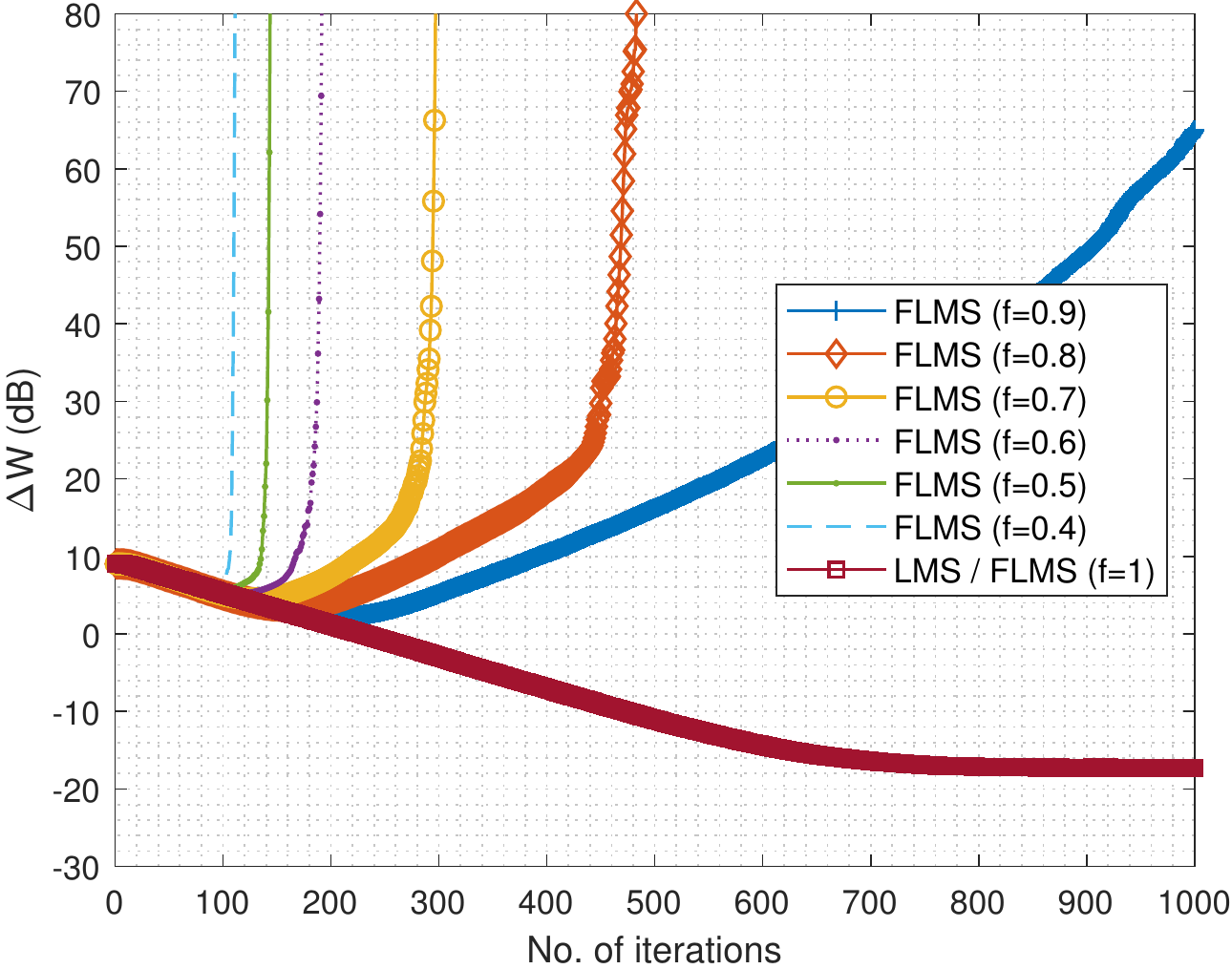}
\caption{Protocol (ii): Learning curves for different values of fractional power for noise-free (left) and noisy (right) environments.}\label{FLMS_pos}
\end{center}
\end{figure}

\subsection{Performance of the Modified Algorithms}

For evaluating the performance of the iterate-update rules \eqref{FLMSRuleComMod} and \eqref{FLMSRule2Mod}, we choose the random desired weights. The desired weight vector is a random signal of length $30$ obtained from a zero-mean Gaussian distribution with a variance of $1$. For each run, new desired weights were selected. We set up both algorithms at an equal convergence rate and compared the steady-state performance. Figure \ref{FLMS_random} shows the learning curves for the LMS and fractional LMS algorithms \eqref{FLMSRuleComMod} and \eqref{FLMSRule2Mod}.   

It can be seen from Fig. \ref{FLMS_random_a} that the fractional algorithm  \eqref{FLMSRuleComMod} shows comparable results with the LMS algorithm but its convergence rate is relatively low for all the listed values of $\alpha$. Based on the experimental result, we can conclude that there is no noticeable gain in using fractional LMS algorithm \eqref{FLMSRuleComMod} while its computational complexity is higher than the LMS algorithm. 

Figure \ref{FLMS_random_b} indicates that the learning rate of the fractional algorithm \eqref{FLMSRule2Mod} is extremely low as compared to the LMS due to the fractional term $|w_n(l)-w_{n-1}(l)|^{1-\alpha}$. This phenomenon is correlated to the Assumption (B3) as well as the stochastic nature of the problem in the protocol (iii). For the stationary case discussed in Example \ref{example2}, the algorithm was performing reasonably better than the stochastic problem in the protocol (iii). On the other hand, the difference $|w_n(l)-w_{n-1}(l)|$ is decreasing as the algorithm is progressing, therefore, the learning rate is decreasing significantly.  To elaborate further on this, we perform the same experiment with $|w_n(l)-w_{n-1}(l) + \epsilon|^{1-\alpha}$ instead of $|w_n(l)-w_{n-1}(l)|$ in the iterate-update rule with bias compensation parameter $\epsilon = 1\times10^{-10}$. We plot the learning curves again in Fig. \ref{FLMS_random_c} where we can see that the learning rate of the algorithm is improved. However, the algorithm is still not converging at a rate better than the LMS. We can conclude based on these experiments that there is no gain in using the fractional algorithm \eqref{FLMSRule2Mod} instead of the conventional LMS algorithm.
\begin{figure}[!htb]
\centering
\subfigure[Rule \eqref{FLMSRuleComMod}]{\includegraphics[width=0.3\textwidth]{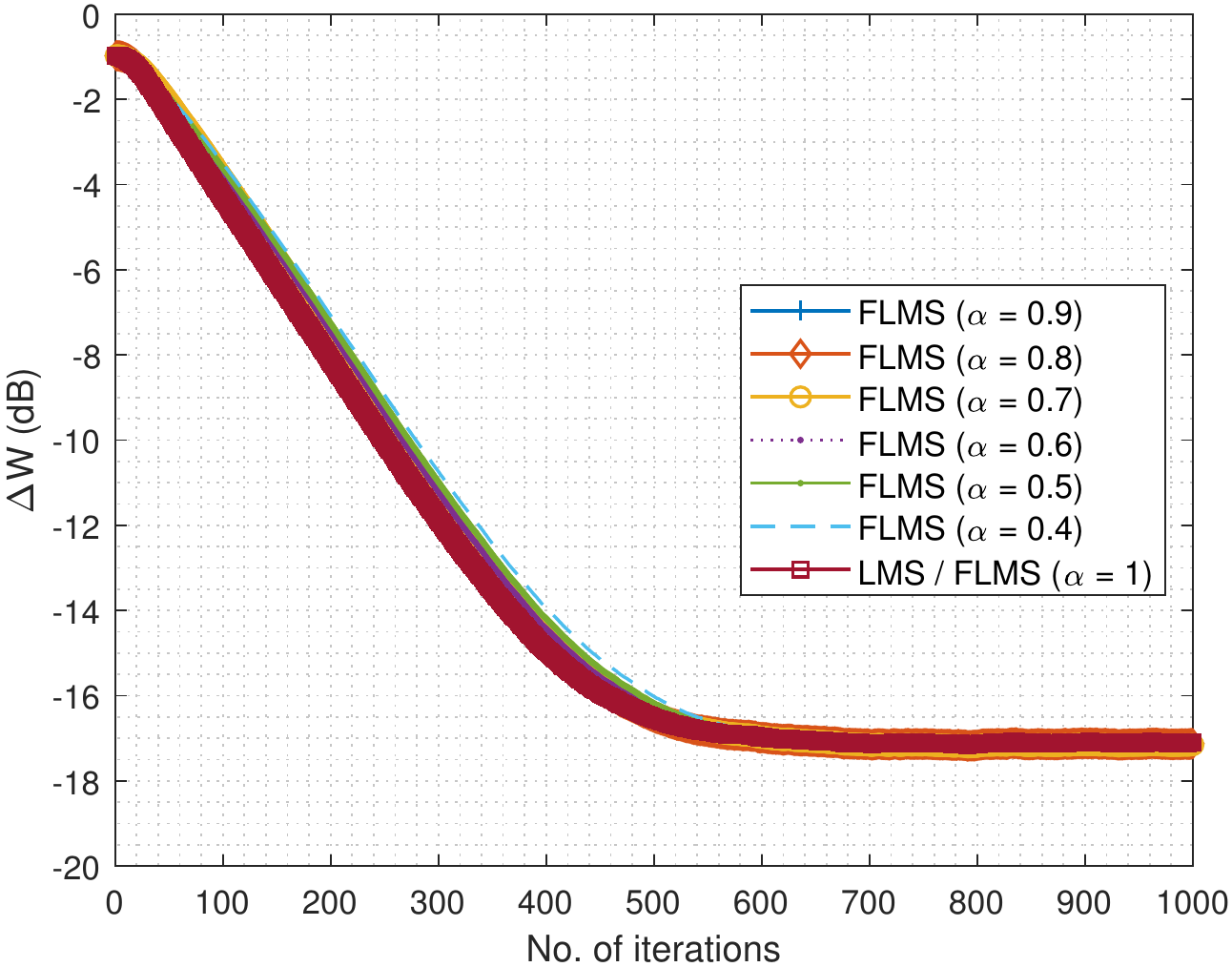}\label{FLMS_random_a}}
\subfigure[Rule \eqref{FLMSRule2Mod}]{\includegraphics[width=0.3\textwidth]{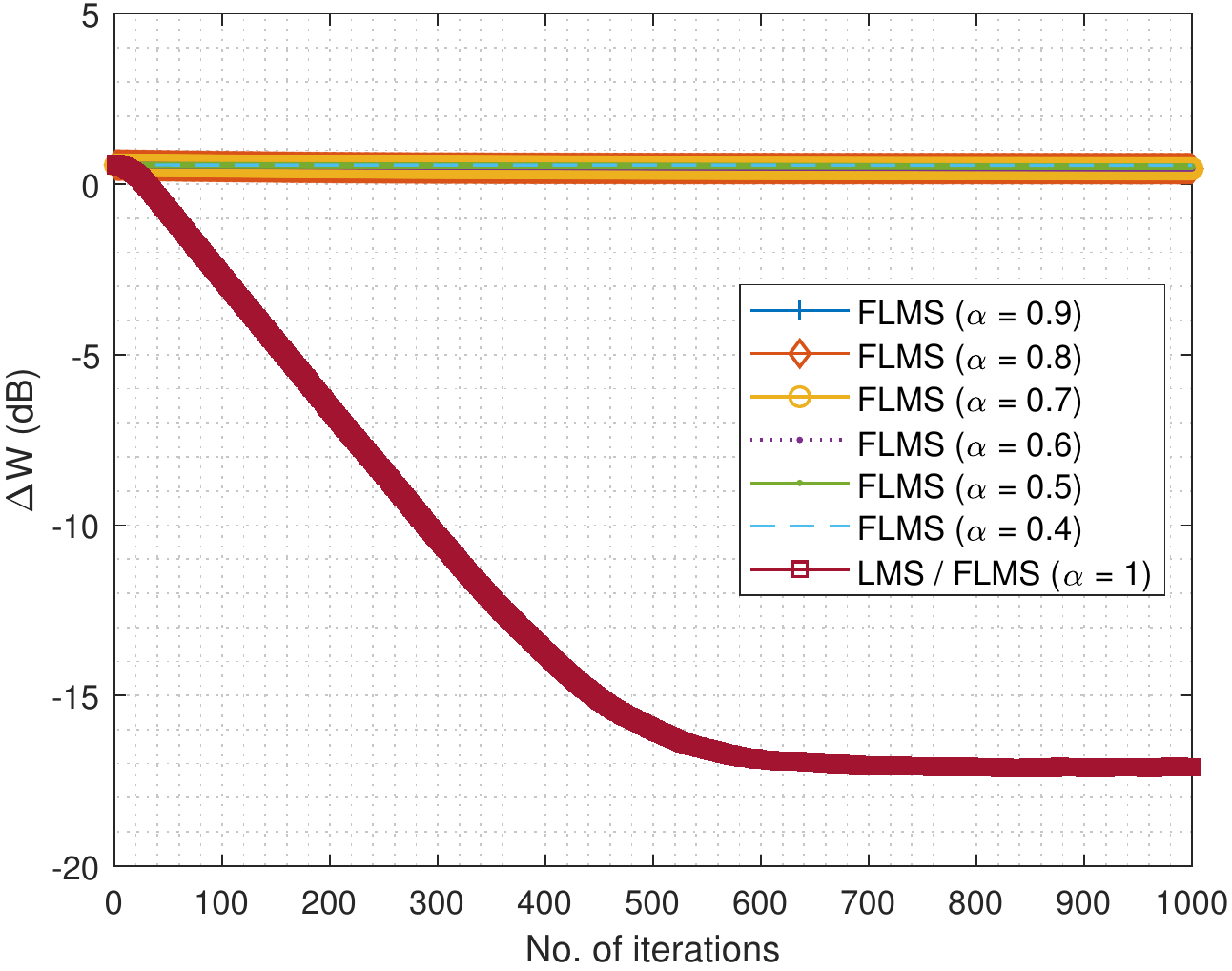}\label{FLMS_random_b}}
\subfigure[Rule \eqref{FLMSRule2Mod} with $\epsilon = 10^{-10}$]{\includegraphics[width=0.3\textwidth]{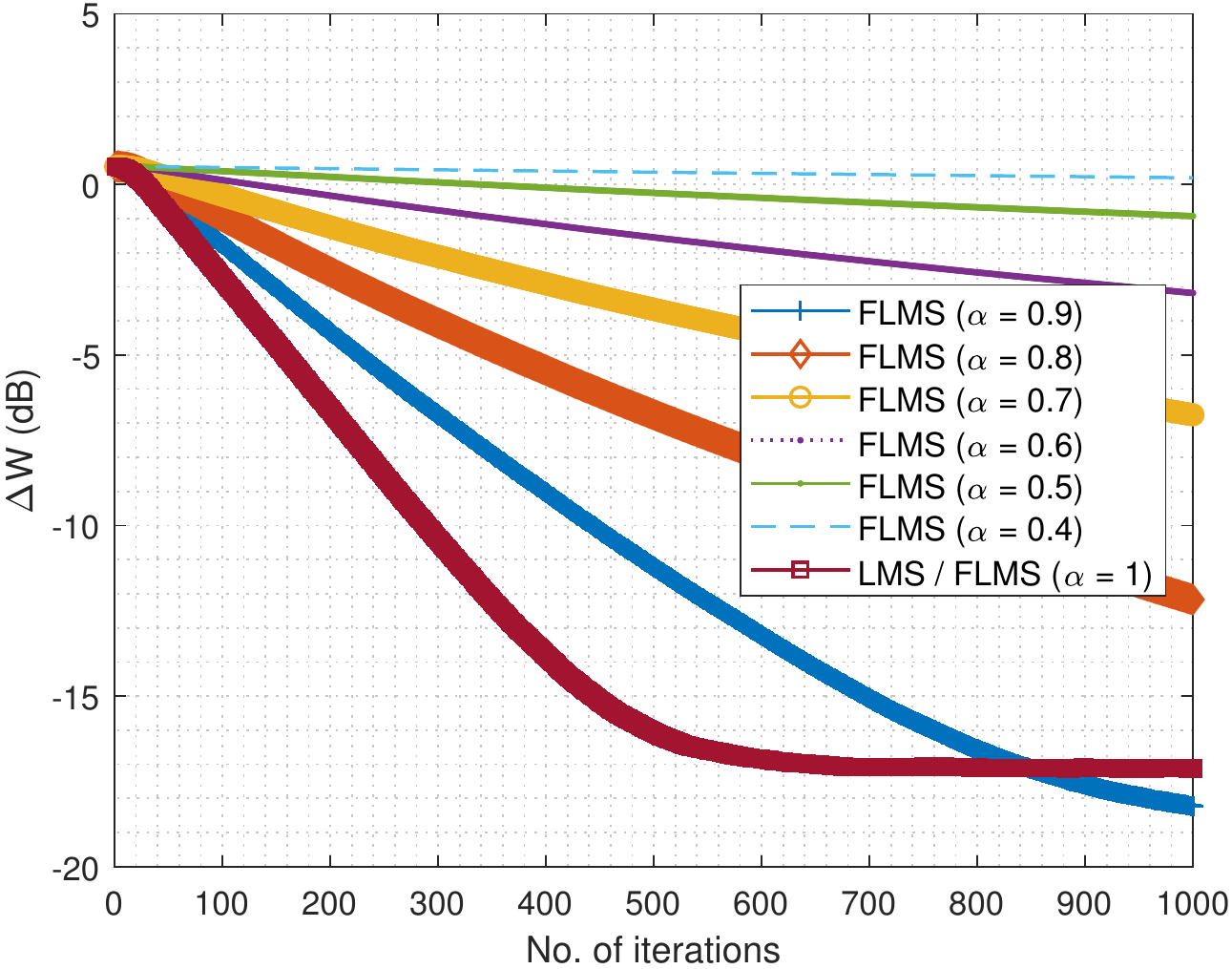}\label{FLMS_random_c}}
\caption{\label{FLMS_random} Protocol (iii): Learning curves for different values of fractional exponent for random weights under noisy environment.} 
\end{figure} 

\section{Conclusion} \label{s:Conc}
In this article, we have rigorously analyzed the performance of the fractional learning algorithms. We have discussed the schematic kinks of the fractional learning algorithms and proposed their remedies. The following are our key observations. 

\begin{enumerate}

\item The fractional gradients are commonly used in literature out of their domain of definition. This is the reason for complex outputs and divergence of the fractional learning algorithms both for negative and positive sought weights.  The use of absolute values in the iterate-update rules to avoid complex outputs can be justified and is a ready-witted remedy.  

\item The fractional gradients in stationary cases do not point opposite to the steepest descent and, therefore, cannot lead to the \emph{steepest} path towards the minimizer of the mean square error. Therefore, their convergence rate cannot be faster than the steepest gradient descent algorithm. Numerical experiments suggest an identical trend in the stochastic case with a possibility of fractional learning algorithms having comparable convergence performance as of the LMS algorithm  due to the stochastic nature of the instantaneous gradient.

\item Since, the geometric interpretation of the fractional derivatives is almost unknown, it is difficult to associate fractional derivatives to the extreme values of a function. In particular, the fractional derivatives are usually nonzero at the true critical point of the mean square error. The fractional critical point is non-unique, dependent on fractional exponents, additive constants, and chosen domain of definition for the fractional derivatives (or roughly the interval in which the solution is sought).

\item The analysis suggests that the performance of the fractional learning algorithms is dependent on the initial guess. In the stationary case, we have proved that the representative fractional iterate-update rules with modulus do converge but not necessarily to the Wiener solution. Their steady-state error, convergence rate, and computational cost cannot be better than the conventional steepest gradient descent. In the stochastic case, numerical simulations indicate that the performance of the fractional algorithms is comparable to that of the LMS algorithm at the best and that at the price of higher computational cost and higher steady-state error. 

\item A rigorous stochastic study of the fractional algorithms requires a complete understanding of the statistical distribution of the fractional iterate-updates, which is still an open question. 

\end{enumerate} 

Based on the rigorous mathematical analysis and numerical experiments performed in this article and the observations above, we conclude that the fractional learning algorithms cannot outperform the conventional integer-order learning algorithms even if their schematic kinks are removed. The fractional learning algorithms have higher computational costs, higher steady-state error, and relatively lower (or comparable at best) convergence rates than their conventional counterparts. Our conclusions conform to those drawn by Bershad, Wen, and So \cite{Bershad} using a comprehensive numerical study.

 \bibliographystyle{plain}

\end{document}